\documentclass[journal]{IEEEtran}
\usepackage{amsmath,amsfonts}
\usepackage{subcaption}
\usepackage{algorithm}
\usepackage{algorithmic}
\usepackage{xcolor}
\usepackage{amsthm}
\usepackage{mathtools}
\usepackage{array}
\usepackage[caption=false,font=normalsize,labelfont=sf,textfont=sf]{subfig}
\usepackage{textcomp}
\usepackage{stfloats}
\usepackage{url}
\usepackage{soul}
\usepackage{verbatim}
\usepackage{graphicx}
\usepackage{cite}
\usepackage{tikz}

\usepackage{tabularx} % For more flexible table columns
\usepackage{multirow} % To allow cells spanning multiple rows
\theoremstyle{plain}
\newtheorem{theorem}{Theorem}
\newtheorem{assumption}{Assumption}
\newtheorem{lemma}{Lemma}
\newtheorem{proposition}{Proposition}

\theoremstyle{definition}

\theoremstyle{remark}
\newtheorem{remark}[theorem]{Remark}

\begin{document}

\title{Decentralized Personalized Federated Learning based on a Conditional ‘Sparse-to-Sparser’ Scheme}

\author{
    Qianyu Long, Qiyuan Wang, Christos Anagnostopoulos,
    Daning Bi
    \thanks{Qianyu Long, Qiyuan Wang, and Christos Anagnostopoulos are with the School of Computing Science, University of Glasgow, Glasgow G12 8QQ, United Kingdom (e-mail: christos.anagnostopoulos@glasgow.ac.uk \& q.long.1@research.gla.ac.uk).}
    \thanks{Daning Bi is with the College of Finance and Statistics, Hunan University Finance Campus, Changsha City, China.}
    %\thanks{Submitted for publication in IEEE Transactions on Neural Networks and
%Learning Systems (TNNLS)}
}

%\markboth{IEEE TRANSACTIONS ON NEURAL NETWORKS AND LEARNING SYSTEMS}%


\maketitle

%\IEEEpeerreviewmaketitle

\begin{abstract}
Decentralized Federated Learning (DFL) has become popular due to its robustness and avoidance of centralized coordination. In this paradigm, clients actively engage in training by exchanging models with their networked neighbors. However, DFL introduces increased costs in terms of training and communication. Existing methods focus on minimizing communication often overlooking training efficiency and data heterogeneity. To address this gap, we propose a novel \textit{sparse-to-sparser} training scheme: 
DA-DPFL. DA-DPFL initializes with a subset of model parameters, which progressively reduces during training via \textit{dynamic aggregation} and leads to substantial energy savings while retaining adequate information during critical learning periods. 

Our experiments showcase that DA-DPFL substantially outperforms DFL baselines in test accuracy, while achieving up to $5$ times reduction in energy costs. We provide a theoretical analysis of DA-DPFL's convergence by solidifying its applicability in decentralized and personalized learning. The code is available at:https://github.com/EricLoong/da-dpfl
\end{abstract}

\begin{IEEEkeywords}
Personalized Federated Learning, Model Pruning, Sparsification, Decentralized Federated Learning.
\end{IEEEkeywords}

\section{Introduction}
Large-scale Deep Neural Networks (DNNs) have gained significant attention due to their high performance on complex tasks. 
The Vision Transformer, ViT-4 \cite{Zhai2022} by Google is a prime example of achieving a new state-of-the-art on ImageNet \cite{Deng2009} with top-1 accuracy of $90.45\%$. 
The success of \textit{centralized} training of DNNs motivated the counterpart \textit{decentralized} training based on Federated Learning (FL) \cite{McMahan2017}. FL involves distributed clients' data in DNN training addressing challenges like privacy \cite{Li2021} by transmitting only model weights and/or gradients instead of raw data. However, FL faces two fundamental challenges \cite{Li2020}: \textit{expensive communication} and \textit{statistical heterogeneity}. Reducing the communication cost due to clients disseminating big-sized DNNs can be achieved by compressing information exchange while attaining model convergence. 
Gradient sparsification and quantization \cite{Sun2019, Lin} significantly reduce communication cost. 
Model pruning \cite{Jiang2022, isik2022sparse,long2023feddip,dai2022dispfl,Li2021b} not only reduces communication cost but also accelerates local training. To alleviate statistical heterogeneity and cope with non-independent and identically distributed (non-i.i.d.) data, Personalized FL (PFL) emerges to allow a local (\textit{personalized}) model \textit{per} client rather than a global one shared among clients. Though PFL is still in its infancy, a plethora of works \cite{ZhangFOMO, dai2022dispfl,Li2021a,Huang2022,Wang2023,huang2022fusion,Li2021b} shows its efficiency in data heterogeneity. 

FL is classified into Centralized FL (CFL) and Decentralized FL (DFL), differentiated by clients' communication methods during training. CFL, exemplified by FedAvg \cite{McMahan2017}, involves a central server coordinating client model aggregation, posing risks of server-targeted attacks and a single point of failure. In contrast, DFL \cite{yuan2023decentralized} offers privacy enhancements and risk mitigation by enabling direct, dynamic, non-hierarchical client interactions within various network topologies such as line/bus, ring, star, or mesh.
% FL is categorized into Centralized FL (CFL) and Decentralized FL (DFL) w.r.t. clients' communication during training. 
% CFL, with FedAvg \cite{McMahan2017} seminal representative, allows 
% a server to communicate and coordinate with all clients for 
% model aggregation. This results in e.g., malicious attacks targeting the server and single point of failure risks. DFL \cite{yuan2023decentralized} mitigates privacy issues and CFL's risks by providing flexibility for clients to interact among themselves, where coordination is achieved under certain networking topology, e.g., line/bus, ring, star, or mesh, without the necessity of a central server (clients connect directly, dynamically and non-hierarchically).  
%  % connect directly, dynamically and non-hierarchically DWFL \cite{chen2022decentralized} is a decentralized wireless federated learning algorithm to satisfy the requirement of Differential Privacy, which converges at a similar rate as centralized training. Depending on link connectivity, \cite{9944194} studied the influence of network topology on the speed-up of decentralized training.
To address communication cost, model/gradient compression-based DFL has been proposed \cite{dai2022dispfl, sun2022decentralized, zhao2022beer,yau2023docom}, where local models are pruned/quantized to achieve competitive performance similar to dense (non-pruned) models.
%\textcolor{red}{This is related work. We need to speak abstractly here and then in the related work we say what it is there and how we depart with.}
%\textcolor{brown}{This is not related work. 17 is the reference to tell the story as a survey.}
%\textcolor{blue}{
%Recently, DisPFL \cite{Dai2022a} and DFedAvgM \cite{sun2022decentralized} stand out as, to our knowledge, the only DFL approaches 
%that eliminate the need of central server for model aggregation. 
%DFedAvgM (a decentralized version of FedAvg) assumes an undirected neighborhood-based topology, where clients communicate only with their direct networked neighbors incorporating momentum, local iterations, and quantization. 
%DisPFL employs personalized sparse masks to customize sparse local models, which attains the target model performance with 
%less communication rounds. DisPFL while equating communication cost for the busiest servers, it incurs considerably higher overall communication training cost (e.g., CO2 emissions) compared to CFL counterparts. 
%This is attributed to the decentralized hybrid topology where each client acts as a server communicating with its neighbors.}
These approaches match communication costs for the busiest servers as CFL, having higher overall communication and training costs attributed to decentralized hybrid topology (each client can act as a server). 
Although DFL with \textit{pointing} protocol, i.e., learning from previously trained models of clients in a sequential line one-peer-to-one-peer, can expedite convergence \cite{yuan2023decentralized}, it struggles with data heterogeneity.
%\textcolor{red}{Here we need to tell the core challenges of DFL--we can take the above approaches' limitations by putting them here by citing them only.}
%In DFL, partial model learning is achieved by clients interacting with their peers given a network topology. While learning from previously trained models in clients can expedite convergence, \cite{yuan2023decentralized} has highlighted challenges in the 
%\textit{pointing-line} DFL strategy, especially when dealing with data heterogeneity. \textcolor{red}{ERIC: We need to tell the challenges of the pointint line strategym and we need to give 1 sentence about what is pointing line, i.e., the locally trained model of a client B depends on that of its preceding client A?}
%While FL frameworks aim at communication and training efficiency, PFL addresses statistical heterogeneity and DFL supports privacy without relying on central servers. However, challenges remain: FL reduces communication cost at the expense of increased training cost associated with gradient/model compression. PFL overlooks communication or training efficiency, while DFL struggles with non-i.i.d. data. 
%This indicates the need for an integrated approach balancing these aspects across different learning paradigms. 
Overall, while FL frameworks target learning from decentralized data, they often overlook either statistical heterogeneity, as in DFL, or efficient training and communication, as in PFL. This highlights the need for an integrated approach that effectively balances communication, training efficiency, and data heterogeneity across different FL paradigms.
%Recently, DisPFL\cite{dai2022dispfl} addresses decentralization and personalization by tackling data heterogeneity. 
We contribute with a novel \textbf{D}ynamic \textbf{A}ggregation \textbf{D}ecentralized \textbf{PFL} framework, coined as \textbf{DA-DPFL}, 
that (i) further reduces communication and training costs, (ii) expedites convergence, and (iii) overcomes data heterogeneity. 
DA-DPFL incorporates two main elements: a fair dynamic scheduling for aggregation of personalized models \textit{and} a dynamic pruning policy. The innovative scheduling policy allows clients in DA-DPFL to \textit{reuse} trained models \textit{within} the same communication round, which significantly accelerates convergence. Moreover, DA-DPFL involves optimized pruning timing to conduct further pruning, i.e., \textit{sparse-to-sparser training}, which does not violate clients' computing capacities while achieving communication, training, and inference efficiency. The trade-off is the controlled latency incurred as some clients await the completion of tasks by their neighbors. We comprehensively assess and compare DA-DPFL with baselines in CFL and DFL to showcase the advantage of dynamic pruning and aggregation in PFL. \\
\textbf{Our major technical contributions are:}
\textbf{(i)} We innovatively align a dynamic aggregation framework to allow clients \textit{reuse} previous 
models for local training within the same communication round. 
\textbf{(ii)} By measuring model compressibility, we propose a \textit{further} pruning strategy, which effectively accommodates and extends existing sparse training techniques in DFL. 
\textbf{(iii)} Compared with both CFL and DFL baselines, % \textbf{FedAvg}\cite{McMahan2017}, \textbf{Ditto}\cite{Li2021a} and \textbf{FedDST} \cite{bibikar2022federated}, and DFL baselines: \textbf{GossipFL}\cite{tang2022gossipfl}, \textbf{DFedAvgM}\cite{sun2022decentralized}, \textbf{DisPFL}\cite{dai2022dispfl}, \textbf{BEER}\cite{zhao2022beer} and \textbf{DFedSAM}\cite{shi2023improving},
our comprehensive experiments showcase that DA-DPFL achieves comparative or even superior model performance across various tasks and DNN architectures. 
\textbf{(iv)} The proposed learning method with \textit{dynamic} aggregation achieves the highest energy and communication efficiency. 
\textbf{(v)} We provide a theoretical convergence analysis, which aligns with experimental observations.

\section{Related Work}
\textbf{Efficient FL:} In distributed ML, communication and training costs are significant challenges. 
LAQ \cite{Sun2019} and DGC \cite{Lin} methods reduce communication costs through gradient quantization and deep gradient compression techniques, respectively. Model compression, notably \textit{pruning}, plays a key role in alleviating device storage constraints, as demonstrated by PruneFL \cite{Jiang2022}, FedDST \cite{bibikar2022federated}, and FedDIP \cite{long2023feddip}, which achieve sparsity in model pruning. 
pFedGate \cite{chen2023efficient} addresses the challenges by adaptively learning sparse local models with a trainable gating layer, enhancing model capacity and efficiency. FedPM \cite{isik2022sparse} and FedMask \cite{Li2021b} focus on efficient model communication using probability masks; with FedMask providing personalized, sparse DNNs for clients and FedPM employing Bayesian aggregation. However, while significant advancements have been made in reducing communication costs \cite{isik2022sparse, Sun2019, Lin, Li2021b}, 
only a few methods \cite{Jiang2022, long2023feddip,chen2023efficient} address reducing training costs through sparse mask learning.

\textbf{Personalized FL:} In FL, addressing data heterogeneity necessitates personalization of global model as achieved by e.g., FedMask \cite{Li2021b}, FedSpa \cite{Huang2022}, and DisPFL \cite{dai2022dispfl} using personalized masks. 
Ditto \cite{Li2021a} offers a fair personalization framework through global-regularized multitask FL, 
while FOMO \cite{ZhangFOMO} focuses on first-order optimization for personalized learning. 
FedABC \cite{Wang2023} employs a `one-vs-all' strategy and binary classification loss for class imbalance and unfair competition, 
while FedSLR \cite{huang2022fusion} integrates low-rank global knowledge for efficient downloading during communication. 
However, such approaches increase training costs highlighting the need for more efficient training methods.

\textbf{Decentralized FL:} Since the work \cite{lalitha2018fully}, DFL emerged as a robust distributed learning paradigm, enabling clients to collaboratively train models with their neighbors, thereby enhancing privacy and reducing reliance on central servers. 
In DFL, increased client interaction leads to methods like DFedAvgM \cite{sun2022decentralized}, which extends FedAvg to decentralized context with momentum SGD, and BEER \cite{zhao2022beer} for non-convex optimization that enhances convergence through communication compression and gradient tracking. GossipFL \cite{tang2022gossipfl} uses bandwidth information to create a gossip matrix allowing communication with one peer using sparsified gradients, reducing communication. DFedSAM \cite{shi2023improving} considers utilizing Sharpness-Aware-Minimization (SAM) optimizer, while DisPFL \cite{dai2022dispfl} utilizes RigL-like pruning in decentralized sparse training to lower generalization error and communication costs. 
%DA-DPFL's performance is compared against from the related work evidencing its applicability and efficiency in a distributed learning context. 
%\textcolor{red}{We do not say how DA-DPFL actually differentiates from the Related Work. We need 2-3 sentences here listing this AND also simply mentioning that we compare DA-DPFL with the baselines/DFL approaches XXX, YYYY, ZZZ in the recent literature showcasing the capability of DA-DPFL to address the challenges we mentioned above.}

\begin{figure*}[htbp]
    \begin{center}
        \includegraphics[width=0.8\textwidth]{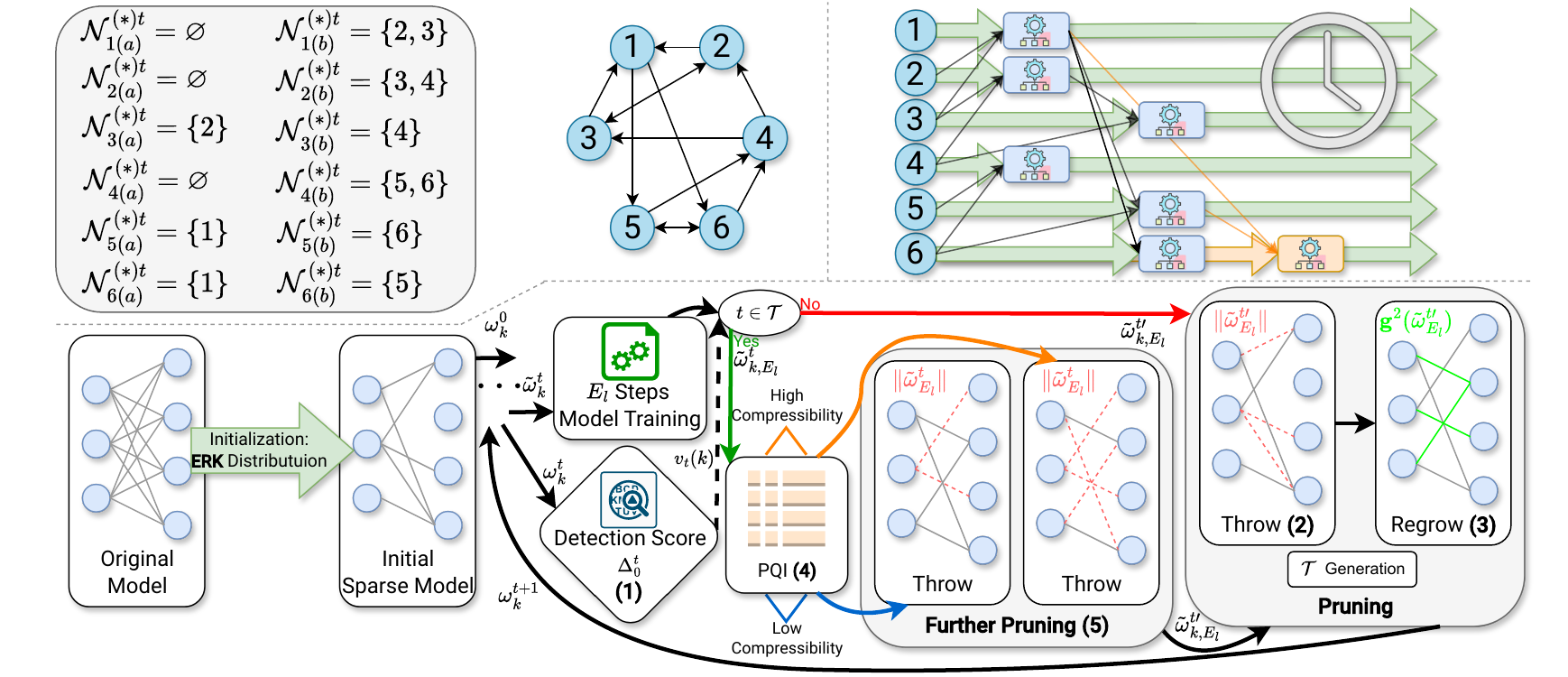}    
    \caption{\textbf{(Top)} Client network ($K=6, M=2,N=1$) with \textit{reuse indexes} $\mathcal{N}_{k(a)}^{(*)t}$ and $\mathcal{N}_{k(b)}^{(*)t}$. Learning schedule: while $N=0$, all nodes train in parallel, i.e., $\mathcal{N}_{k(a)}^{(*)t}=\emptyset$; if $N=1$, node 3 waits for 2, node 5 and 6 for 1; nodes 1, 2, 4 begin parallel training immediately; $N=2$ enables node 6 wait for $1$ and 5, marked with \textit{different} color. \textbf{(Bottom)} Training process at time $t$ for client $k$. Flow follows $t\in \mathcal{T}$: `\textcolor{red}{no}' leads to normal \textit{sparse} training, `\textcolor{green}{yes}' to proposed \textit{sparser} training. Steps: (1) Detection score calculation using $\omega_{k}^{t}$, determining $t^{*}$; (2) Magnitude-based weight pruning; (3) Gradient-flow-driven weight recovery; (4) PQI evaluation for NN compressibility; (5) Additional pruning based on compressibility level.}
    \label{fig:network}
    \end{center}
\end{figure*}

\section{Problem Fundamentals \& Preliminaries}
%\subsection{Time-varying Topology} 
Consider a distributed system with $K$ clients indexed by $\mathcal{K} = \{1, 2, \ldots, K\}$. The clients are 
networked given a topology represented by a graph $\mathcal{G}(\mathcal{K}, \mathbf{V})$, where the adjacency matrix 
\( \mathbf{V}=[v_{i,j}] \in \mathbb{R}^{K\times K} \) \cite{sun2022decentralized} 
%exhibits the properties: \textit{symmetry}, \textit{null space property}, and \textit{spectral property}. 
defines the neighborhood $\mathcal{G}_{k}$ of client $k \in \mathcal{K}$, i.e., subset of clients that directly communicate with  client $k$, $\mathcal{G}_{k} = \{i \in \mathcal{K}: v_{i,k} > 0\}$. 
An entry $v_{i,k} = 0$ indicates no communication from client $i$ to client $k$, i.e., $i \notin \mathcal{N}_{k}$. 
Note that \( v_{i,k}=v_{k,i}\) may not always be valid for \( i\neq k \).
The topology can be static or dynamic. 
In our case, we adopt dynamic communication among clients, i.e., entries in $\mathbf{V}^{t}$ depend on (discrete) time instance $t \in \mathbb{T} = \{1, 2, \ldots\}$.
We define a time-varying and non-symmetric network topology via \( \mathbf{V}^{t}=[v_{i,j}^{t}] \in \mathbb{R}^{K\times K} \) accommodating temporal neighborhood \( \mathcal{G}_{k}^{t} \) for $k$-th client. %It may hold \( v_{i,k}^{t} \neq v_{k,i}^{t} \) for \( i\neq k \).
%Non-zero entries in $\mathbf{V}^{t}$ signify communication between clients exchanging model parameters during training. 
%\subsection{Problem Formulation}
We consider a scalable DFL setting with \( K \) clients (e.g., mobile devices, IoT devices)
with a time-varying topology. 
Each client \( k \in \mathcal{K} \) possesses local data \( \mathcal{D}_{k} = \{(\mathbf{x},y)\}\) of input-output pairs $\mathbf{x} \in \mathcal{X}$ and $y \in \mathcal{Y}$, and communicates with neighbours \( \mathcal{G}_{k}^{t} \) exchanging models. 
The problem formulation of PFL (adopting the formulation in \cite{dai2022dispfl})
%m, as described in \cite{dai2022dispfl}, 
seeks to find the models $\mathbf{\omega}_{k}, \forall k \in \mathcal{K}$, that minimize:
\begin{eqnarray}
\label{eq:optim_PFL}
\min_{\{\mathbf{\omega}_{k}\}, k\in[1,K]} f(\{\mathbf{\omega}_{1}\}_{k=1}^{K}) = \frac{1}{K}\sum_{k=1}^{K} F_{k}(\mathbf{\omega}_{k}),
%\end{equation}
%\begin{equation}
\end{eqnarray}
where $F_{k}(\mathbf{\omega}_{k}) = \mathbb{E}[\mathcal{L}(\mathbf{\omega}_{k}; (\mathbf{x},y)) | (\mathbf{x},y) \in \mathcal{D}_{k}]$ with expected loss function \( \mathcal{L}(.;.) \) between actual and predicted output given local data \( \mathcal{D}_{k} \).
We depart from (\eqref{eq:optim_PFL}) by adopting model pruning in DFL aiming at eliminating 
non-essential model weights. This is achieved by utilizing a binary mask \(\mathbf{m}\) over a model. 
Hence, the pruning-based (masked) PFL problem is succinctly formulated as finding the global model $\mathbf{\omega}$ and individual masks $\mathbf{m}_{k}, \forall k \in \mathcal{K}$:
\begin{eqnarray}
\label{eq:optim_MPFL}
\min_{\mathbf{\omega}, \{\mathbf{m}_{k}\}, k\in[1,K]} f(\{\mathbf{\omega}_{k}\}_{k=1}^{K}) = \frac{1}{K}\sum_{k=1}^{K} F_{k}(\omega\odot \mathbf{m}_{k}),
%\end{equation}
%\begin{equation}
%\label{eq:localossmask}
\end{eqnarray}
where $F_{k}(\omega\odot \mathbf{m}_{k}) = \mathbb{E}[\mathcal{L}(\omega\odot \mathbf{m}_{k}; (\mathbf{x},y)) | (\mathbf{x},y) \in \mathcal{D}_{k}]$; \( \odot \) represents the Hadamard product (element-wise product) of two matrices. 
The individual mask \( \mathbf{m}_{k} \) denotes a pruning operator specific to client \( k \). 
Given mask $\mathbf{m}_{k}$, the sparsity $s_{k} \in [0,1]$ of \( \mathbf{m}_{k} \) indicates the proportion of non-zero model weights among all weights. 
%Following the definition of masked PFL presented in \cite{dai2022dispfl}, 
The goal of \eqref{eq:optim_MPFL} is to seek a global model \( \boldsymbol{\omega} \) and individual masks \( \mathbf{m}_{k} \) such that the optimized personalized model for \textit{each} client $k \in \mathcal{K}$ is given by \( \boldsymbol{\omega}_{k} =\boldsymbol{\omega}\odot \mathbf{m}_{k} \), while clients communicate at time $t$ only with their neighbors $\mathcal{G}_{k}^{t}$ given  the time-varying $\mathbf{V}^{t}$.

\section{The DA-DPFL framework}
\subsection{Overview}
We introduce the DA-DPFL framework to tackle the problem in Eq. \eqref{eq:optim_MPFL}. 
DA-DPFL not only addresses data heterogeneity efficiently via masked-based PFL but also significantly improves convergence speed by incorporating a fair dynamic communication protocol. Sequential pointing line communication adopts an one-peer-to-neighbors mechanism striking the balance between computational parallelism and delay. As described in \cite{yuan2023decentralized}, two sequential pointing DFL strategies, \textit{continual} and \textit{aggregate}, facilitate knowledge dissemination in distributed learning. %Both strategies are predicated on \textit{waiting} for aggregating the model sequentially trained by the preceding client. As evidenced in \cite{yuan2023decentralized}, when all $K$ clients complete their training (in a single round) over i.i.d. data, these strategies demand less communication and computation to achieve similar model performance compared to that using CFL (in a single round). This is because subsequent (succeeding) clients refine their models based on the aggregated models of their preceding peers. 
However, certain challenges are evident as discussed in \cite{yuan2023decentralized} such as data heterogeneity and non-scalability.
%(L1) The waiting time for a single round is $K$ times higher than that of CFL.
%(L2) The local models on $K$ clients are assumed to be drawn from multivariate normal distribution in parameter space. 
%(L3) Challenges arise from data heterogeneity; non-i.i.d. data are detrimental to model performance.  
%(L4) Scalability is not achieved; pointing line strategies require $O(K)$ passes per round.
%Moreover, clients at the end of the sequence risk to \textit{forget} the derived knowledge from preceding clients due to the extended sequential order (especially when $K$ becomes large). This is not unusual in cross-device learning environments. 
Furthermore, the decision on when to apply pruning during training is crucial \cite{frankle2019stabilizing, pmlr-v162-rachwan22a}. While \textit{early} pruning reduces computational cost, it may adversely affect performance. Choosing an optimal pruning time can enhance both training and communication efficiency, often with minimal performance degradation or even improvements. As elaborated in \cite{dai2022dispfl} and \cite{hoefler2021sparsity}, a sparser model tends to have a reduced generalization bound characterized by smaller discrepancy between training and test errors. Finding sparser models with lower generalization error, DA-DPFL introduces an innovative dynamic pruning strategy. 
DA-DPFL addresses data heterogeneity while decreasing the number of communication rounds needed for convergence and achieving high model performance. This comes at a small and controllable delay in 
learning from the trained models. The processes of DA-DPFL are depicted in Fig. \ref{fig:network} and Algorithm \ref{algorithm:adpfl}. %To align with literature, e.g., \cite{dai2022dispfl}, we define one communication round as \textit{the completion time of all clients involved in training}. 

\begin{remark}
The relationship between model sparsity and performance links to the complexity of the task and model architecture. Pruning becomes a necessary solution when there is model redundancy for the given task, aligned with the findings in \cite{dai2022dispfl} and \cite{hoefler2021sparsity}. If the model is non-redundant, an increased sparsity invariably affects model performance. 
\end{remark}

\begin{algorithm}[tb]
\caption{The DA-DPFL Algorithm}
\label{algorithm:adpfl}
\begin{algorithmic}[1]
\STATE \textbf{Input:} $K$ clients; $T, E_{l}$ rounds; %initial cosine annealing ratio $\alpha_{0}$; 
PQI hyper-param. $\{p,q,\gamma,\eta_{c}\}$, pruning thr $\delta_{pr}$; voting threshold $\delta_{v}$; factors $b, c$; target sparsity $s^{*}$;
\STATE \textbf{Output:} Personalized aggregated models $\{\boldsymbol{\tilde{\boldsymbol{\omega}}}_{k}^{T}\}_{k=1}^{K}$. 
\STATE \textbf{Initialization:} Initialize $\{\mathbf{m}_{k}^{0}\}_{k=1}^{K}$, $\{\boldsymbol{\omega}_{k}^{0}\}_{k=1}^{K}$, $\mathcal{T} \leftarrow \emptyset$
\FOR{round $t = 1$ \textbf{to} $T$}
    \FOR{each client $k$}
        \STATE Generate a random \textit{reuse} index set $\{\mathcal{N}^{t}_{k}\}_{k=1}^{K}$.
        \STATE Generate a random bijection $\pi^{t}_{k}$ between $\mathcal{N}^{t}_{k}$ and $\mathcal{G}^{t}_{k}$
        \STATE Form prior and posterior set $\{\mathcal{N}^{(*)t}_{k(a)},\mathcal{N}^{(*)t}_{k(b)}\}$
        \STATE Form  $\{\mathcal{G}^{(*)t}_{k(a)},\mathcal{G}^{(*)t}_{k(b)}\}$ by $\{\mathcal{N}^{(*)t}_{k(a)},\mathcal{N}^{(*)t}_{k(b)},\pi^{-1(t)}_{k}\}$

        \IF{$\mathcal{G}_{(a)k}^{(*)t}!= \emptyset$}
            \STATE \textbf{do} Wait models from neighbors $\mathcal{G}_{(a)k}^{(*)t}$
            %\STATE \textbf{while} (receive from slowest client in $\mathcal{N}_{(a)k}^{t}$)
        \ENDIF
        \STATE Receive neighbor's models $\boldsymbol{\omega}_{j}^{t}, j \in \mathcal{G}_{k}^{t}$ 
        \STATE Obtain mask-based aggregated model $\boldsymbol{\Tilde{\omega}}_{k}^{t}$.
        \STATE Compute $\boldsymbol{\Tilde{\omega}}_{k,\tau}^{t}$ for $E_{l}$ local rounds.
        \STATE Calculate $\Delta_{0}^{t}(k)$ and $v_{t}(k)$ based on $\delta_{pr}$.
        \STATE Broadcast $v_{t}(k)$ to all clients; derive $t^{*}$.
        \IF{$t \in \mathcal{T}$ \textbf{and} $s_{k} < s^{*}$}
            \STATE %Apply PQI pruning to obtain 
            Call Algorithm \ref{algorithm:SAP} to obtain $\boldsymbol{\Tilde{\omega}}_{k,E_{l}}^{t\prime}, \mathbf{m}_{k}^{t\prime}$ 
            \STATE Update sparsity $s_{k}$
        \ELSE
            \STATE Set $(\boldsymbol{\Tilde{\omega}}_{k,E_{l}}^{t\prime}, \mathbf{m}_{k}^{t\prime}) \leftarrow (\boldsymbol{\Tilde{\omega}}_{k,E_{l}}^{t}, \mathbf{m}_{k}^{t})$
        \ENDIF
        \STATE Call Algorithm \ref{algorithm:rigl} to update $\mathbf{m}_{k}^{t+1}$ 
        \STATE Set $\boldsymbol{\omega}_{k}^{t+1}=\boldsymbol{\Tilde{\omega}}_{k,E_{l}}^{t\prime}$ 
    \ENDFOR
    %\IF{$t == t^{*}$}
     \STATE \textbf{if} $t == t^{*}$ \textbf{then} Update $\mathcal{T}$.
    %\ENDIF
\ENDFOR
\end{algorithmic}
\end{algorithm}

\subsection{Learning Scheduling Policy}
In this section, we outline the scheduling policy adopted for client participation within our framework. This applies to any topological connection, such as a \textit{ring} or \textit{fully-connected} network, where neighborhood sets, denoted as $\mathcal{G}_{k}^{t},k\in\{1, 2, \ldots, K\}$, are established. It is important to note that DA-DPFL is particularly suited for a time-varying connected topology while remains flexible to accommodate a static topology, represented as $\mathcal{G}_{k}$. At the start of each communication round, denoted by $t$, \textit{reuse indexes} for neighborhood sets, $\mathcal{N}_{k}^{t}$, are randomly assigned to $M$ clients, where $\lvert \mathcal{G}_{k}^{t} \rvert=\lvert \mathcal{N}_{k}^{t} \rvert = M < K$ and $\mathcal{G}_{k}^{t} \stackrel{\pi_{k}^{t}}{\leftrightarrow} \mathcal{N}_{k}^{t}$, where $\pi_{k}^{t}$ is a random bijection mapping. Given $\mathcal{G}_{k}^{t}$ and $\mathcal{N}_{k}^{t}$ are both discrete sets,
\begin{align}\label{eq:pi_definition}
    i = \pi_{k}^{t}(j),
\end{align}
with index $i \in \mathcal{N}_{k}^{t} $ and $j \in \mathcal{G}_{k}^{t}$.
%This process is independent of the formation of $\mathcal{G}_{k}^{t}$. 
It is crucial to acknowledge that $\mathcal{N}_{k}^{t}$ may be equal with $\mathcal{G}_{k}^{t}$ if sets are randomly generated. For simplicity, we let $\mathcal{N}_{k}^{t}= \mathcal{G}_{k}^{t}$ in Fig. \ref{fig:network}, where client $k$ is indexed with reuse index $k$. 
The introduction of $\mathcal{N}_{k}^{t}$ serves to emphasize the independence in the generation of \textit{reuse} indexes, which are pivotal in guiding the dynamic aggregation process.
\textbf{Note:} the criteria for establishing $\mathcal{G}_{k}^{t}$ are influenced by factors e.g., network bandwidth, geographical location, link availability; however, $\mathcal{N}_{k}^{t}$ is independent of these factors. We reassign client indices for each training round.

 %thereby constituting a random time-varing connection $\mathbf{V}^{t}$. This stochastic approach to topology formation ensures equitable opportunities for each client to acquire models from their neighbors, effectively mitigating the challenges associated with sequential, directed strategies often encountered in learning processes.

%\textcolor{blue}{For each time step $t \in \mathbb{T}$, every $k^{\text{th}}$ client possesses knowledge of its \textit{reuse} neighborhood index $\mathcal{N}_{k}^{t}$, defined such that the neighborhood size  Note that the criteria for establishing $\mathcal{G}_{i}^{t}$ could be influenced by factors like network bandwidth, geographical location, and link availability, but $\mathcal{N}_{k}^{t}$ is independent to them. }%However, these considerations fall outside the purview of this paper, and as such, neighborhoods $\mathcal{N}_{k}^{t}$ are construed randomly for fairness.

Within DA-DPFL, a client $k$ may defer the reception of models from \textit{some} neighbors, contingent upon $\mathcal{N}_{k}^{t}$. We denote $\mathcal{N}_{k}^{t}$ for each client $k$ into two subsets based on the reuse indices of neighboring clients: 
(a) a \textit{prior} client subset $\mathcal{N}_{(a)k}^{t} = \{n_{k}^{t} \leq k : n_{k}^{t} \in \mathcal{N}_{k}^{t}\}$, and 
(b) a \textit{posterior} client subset $\mathcal{N}_{(b)k}^{t} = \{n_{k}^{t} > k : n_{k}^{t} \in \mathcal{N}_{k}^{t}\}$ (refer to Fig. \ref{fig:network} Top). Should $\mathcal{N}_{(b)k}^{t} = \emptyset$, implying $\mathcal{N}_{(a)k}^{t} = \mathcal{N}_{k}^{t}$, client $k$ awaits the \textbf{slowest} client within $\mathcal{N}_{k}^{t}$ before commencing model aggregation and local dataset training $\mathcal{D}_{k}$. To enhance scalability, we introduce a threshold to allow waiting for, at most, $N$ fastest clients in $\mathcal{N}_{(a)k}^{t}$, where $\lvert \mathcal{N}_{(a)k}^{(*)t}\rvert =N\leq M$. Then, $\mathcal{N}_{(a)k}^{(*)t}\cup \mathcal{N}_{(b)k}^{(*)t}=\mathcal{N}_{k}^{t}$. Conversely, absence of a prior client set ($\mathcal{N}_{(a)k}^{(*)t} = \emptyset$) enables client $k$ to incorporate models from $\mathcal{N}_{(b)k}^{(*)t}$ \textit{without} delay, as illustrated by nodes 1, 2, and 4 in Fig. \ref{fig:network}. Based on the bijection mapping between $\mathcal{N}_{k}^{t}$ and $\mathcal{G}_{k}^{t}$, $\mathcal{G}_{(a)k}^{(*)t}$ is obtained. %Specifically, client 1, having no preceding clients, initiates without delay.
Therefore, DA-DPFL achieves a hybrid scheme between continual learning with delayed aggregation and immediate aggregation, i.e., dynamic aggregation. Continual learning is achieved by gradual learning of the models from clients in client $k$'s prior set. The benefit obtained is the sequential knowledge transfer from clients in the prior set. %(forming a `logical' pointing line learning), which is evidenced to significantly speed up the convergence rate. 
This comes at the expense of a potential delay to client $k$ for aggregating the models from $\mathcal{G}^{(*)t}_{k(a)}$. On the other hand, the models of the clients from the posterior set are independently sent to client $k$, without any delay achieving training parallelism.%At the next round $t+1$, a new neighborhood $\mathcal{G}_{k}^{(t+1)}$ is formed for each client $k$ ensuring fairness in peer collection and breaking any ties of clients being potentially engaged in consecutive rounds. 

\begin{remark}
    %\textbf{Distinctiveness from Sequential and Parallel FL:} 
    DA-DPFL learning schedule diverges from traditional FL paradigms.
    %by facilitating partially parallel learning. 
    In our example, at time $t$, nodes $\{1,2,4\}$ engage in simultaneous (parallel) training, while nodes $\{3,5,6\}$ await model reuse from preceding clients. This methodology 
    %not only embodies partial parallelism but also 
    allows subsequent nodes to train concurrently with preceding ones, as shown by nodes $\{5,6\}$ training in tandem with node $3$. 
    The introduction of a cutoff value $N$ endows our waiting policy with controllability. 
    If $N=0$, DA-DPFL operates as a parallel FL system with sparse training; while $N=M=K$ transits DA-DPFL to a \textit{sequential} FL. %For instance, that node $6$ can commence training alongside node $5$ without necessitating a waiting period, provided $N=1$.
\end{remark}

\subsection{Time-optimized Dynamic Pruning Policy}
Alongside scheduling of local training and gradual model aggregation achieved by prior and posterior neighbors per client, we introduce a dynamic pruning policy. The initial mask $\mathbf{m}_{k}^{0}, \forall k$, is set up in accordance with the Erdos-Renyi Kernel (ERK) distribution \cite{evci2020rigging}. Subsequently, masks are removed and re-grown based on the importance scores, which are computed from the magnitude of model weights and gradients. This strategy is an extension of the centralized RigL \cite{evci2020rigging} to DA-DPFL as elaborated in Appendix \ref{sec:rigl} (line:24) in Algorithm \ref{algorithm:adpfl}. We devise a method that is orthogonal to other fixed-sparsity training methods like %GraSP \cite{wang2020picking}, SNIP \cite{lee2018snip} and 
RigL facilitating \textit{further} pruning.
The Sparsity-informed Adaptive Pruning (SAP) in \cite{diao2022pruning} introduces the PQ Index (PQI) to 
assess the potential `compressibility' of a DNN (line:19; Appendix \ref{sec:appendix_SAP}). 
DA-DPFL leverages PQI by integrating within DFL, which addresses the heterogeneity of various local models by adaptively pruning different models. 
%Preliminary experiments demonstrate that ADPFL sophisticatedly combining PQI and RigL significantly outperforms RigL alone when targeting equivalent sparsity levels. 
%\textbf{Note:} details on PQI (line:22) and RigL (line:) algorithms n line:26) in Algorithm \ref{algorithm:adpfl}).% along with experimental results of ADPFL combining both PQI and RigL (PQI+RigL) are reported in Appendix \ref{sec:appendix_SAP}.
In centralized learning, EarlyCrop's analysis \cite{pmlr-v162-rachwan22a} on pruning scheduling relies on sufficient information during \textit{critical learning periods}, while CriticalFL \cite{yan2023criticalfl} advocates for an early doubling of information transmission. EarlyCrop leverages between \textit{gradient flow} and \textit{neural tangent kernel} to facilitate seamless transition into model pruning. We adjust the pruning time detection score as:
\begin{equation}\label{eq:pruning_time_1}
    \frac{\lvert \Delta_{0}^{t}-\Delta_{0}^{t-1}\rvert}{\lvert\Delta_{0}^{1}\rvert}<\delta_{pr}; \Delta_{0}^{t}:=\|\boldsymbol{\omega}^{t}-\boldsymbol{\omega}^{0}\|^{2},  
\end{equation}
where \(\delta_{pr}\) is a predefined threshold. %It is not trivial to estimate suitable values for \(\delta_{pr}\) given the diversity of model architectures. 
In DA-DPFL with \(K\) clients, we introduce a \textit{voting majority rule}, where the client \(k\)'s vote is defined as:
\begin{equation}
v_{t}(k) = 
\begin{cases} 
1 & \text{if } \frac{\lvert \Delta_{0}^{t}(k)-\Delta_{0}^{t-1}(k)\rvert}{\lvert\Delta_{0}^{1}(k)\rvert}<\delta_{pr}, \\
0 & \text{otherwise.}
\end{cases}
\end{equation}
Hence, the first time to prune is determined by $K$ clients as: 
%\begin{equation}
%\label{eq:optimal_t_prune}
    $t^{*}=\min\{ t : \frac{1}{K}\sum_{k=1}^{K}v_{t}(k) < \delta_{v}\}$
%\end{equation} 
where \(\delta_{v}\) represents the ratio threshold for voting. Given the first pruning time $t^{*}$, DA-DPFL determines the frequency of pruning for rounds $t>t^{*}$. Based on the influence of early training phase, a.k.a. \textit{critical learning} period \cite{jastrzebski2021catastrophic} on the local curvature of the loss function in DNNs, our strategy permits a low pruning frequency during the initial stages which intensifies pruning as model approaches convergence. This balance between communication overhead and model performance yields an optimal pruning frequency that varies across tasks and model architectures. We define the \textit{pruning frequency}, i.e., the gap between consecutive pruning events, and in turn the pruning times by non-evenly dividing the rest of the horizon $T-t^{*}$:
\begin{equation}
\label{eq:pruing_time_3}
    I_{\tau}:= \lceil \frac{t^{*}+b}{c^{\tau-1}} \rceil, \tau \in \{\mathbb{Z}_{\geq1}\}.
\end{equation} 
Parameter $b > 0$ delays the optimal first pruning time, $c>0$ is a scaling factor to adjust pruning frequency. The $p$-th pruning time $t_{p}$ with $t_{p} > t^{*}$, is $t_{p}=\sum_{\tau=1}^{p}I_{\tau}$ obtaining the pruning times set $\mathcal{T} = \{t_{1}, \ldots, t_{p}: t^{*} < t_{p} < T\}$. 

\subsection{Masked-based Model Aggregation}
For notation compatibility, following the model aggregation operator in DisPFL\cite{dai2022dispfl} and FedDST \cite{bibikar2022federated}, the client $k$'s aggregated model $\boldsymbol{\Tilde{\omega}}_{k}^{t}$ derived from the models of client $k$'s neighbors in $\mathcal{G}^{t}_{k}$ at round $t$ based on masked local model is:
\begin{equation}
\label{eq:mask_weight_aggr}
    \boldsymbol{\Tilde{\omega}}_{k}^{t}=\Big(\frac{\sum_{j\in \mathcal{G}_{k+}^{t}\boldsymbol{\omega}_{j}^{t}}}{\sum_{j\in \mathcal{G}_{k+}^{t}\mathbf{m}_{j}^{t}}}\Big)\odot \mathbf{m}_{k}^{t},
\end{equation}
where $\mathcal{G}_{k+}^{t} = \mathcal{G}_{k}^{t} \cup \{k\}$ is client $k$'s neighborhood including client $k$. The local training rounds $\tau \in E_{l}$ based on the obtained $\boldsymbol{\Tilde{\omega}}_{k}^{t}$ is:
%\begin{equation}
%\label{eq:local_train}
$\boldsymbol{\Tilde{\omega}}_{k,\tau+1}^{t}=\boldsymbol{\Tilde{\omega}}_{k,\tau}^{t}-\eta (\mathbf{g}_{k,\tau}^{t}\odot\mathbf{\mathbf{m}_{k}^{t}})$,
%\end{equation}
where $\mathbf{g}_{k,\tau}^{t}$ is the gradient of local loss function $F_{k}(\cdot)$ w.r.t. $\boldsymbol{\Tilde{\omega}}_{k,\tau}^{t}$.

\section{Theoretical Analysis}
% We analyze the effect of new scheduling strategy in the convergence analysis in this section.
%Note that the model parameters are sparse, which also satisfies the common assumptions (Assumptions~\ref{assumption:mu_lips}, \ref{assumption:bounded_g} and \ref{assumption:bounded_var}) since these assumptions do not require the density of model parameter $\omega$.
%$\nabla f$ denotes the expected (real) gradients of the corresponding loss function.
\begin{assumption}
\label{assumption:mu_lips}
    \textbf{$\mu$-Lipschitz-continuity:} $\forall \boldsymbol{\omega}_{1}, \boldsymbol{\omega}_{2} \in \mathbb{R}^d$, $\forall k \in \lvert K \rvert, \mu \in \mathbb{R}:
    %\begin{equation}\label{eq:mu_lipschitiz}
        \|\nabla f_{k}(\boldsymbol{\omega}_{1}) - \nabla f_{k}(\boldsymbol{\omega}_{2})\| \leq \mu\|\boldsymbol{\omega}_{1} - \boldsymbol{\omega}_{2}\|$.
    %\end{equation}
\end{assumption}

\begin{assumption}\label{assumption:bounded_var}
    \textbf{Bounded variance for gradients:} \cite{sun2022decentralized} $\forall k \in \lvert K \rvert$ and  $\boldsymbol{\omega}\in \mathbb{R}^{d}$:
    \begin{align}
                \mathbb{E}[\|\nabla \hat{f}_{k}(\boldsymbol{\omega})-\nabla f_{k}(\boldsymbol{\omega})\|^2]\leq \sigma^{2}_{l},\\
                \frac{1}{K}\sum_{k=1}^{K}\|\nabla f_{k}(\boldsymbol{\omega})-\nabla f(\boldsymbol{\omega})\|^{2} \leq \sigma^{2}_{g},\\
                \frac{1}{K}\sum_{k=1}^{K}\|\nabla \tilde{f}_{k}(\boldsymbol{\omega\odot \mathbf{m}_{k}})-\nabla f(\boldsymbol{\omega})\|^{2} \leq \sigma^{2}_{p},
    \end{align}
    $\hat{f}(\cdot)$ is the estimated gradients from training data; $\tilde{f}(\cdot)$ is personalized global gradients. 
\end{assumption}
\begin{assumption}\label{assumption:mask_influence}
    %Let the personalized mask for client \(k\) influence the aggregation step in DFL and
    The aggregated model \(\Tilde{\boldsymbol{\omega}}_{k}^{t}\) for client \(k\) at iteration \(t\) is given by:
    \begin{align}
        \Tilde{\boldsymbol{\omega}}_{k}^{t} &= \left(\frac{\sum_{j\in \mathcal{G}_{k}^{t}}\boldsymbol{\omega}_{j}^{t}}{\sum_{j\in \mathcal{G}_{k}^{t}}\mathbf{m}_{j}^{t}}\right)\odot \mathbf{m}_{k}^{t} = \left(\frac{\sum_{j\in \mathcal{G}_{k}^{t}}\boldsymbol{\omega}_{j}^{t}}{M}\right)\odot \mathbf{m}_{k}^{t}
%        &= \Bar{\boldsymbol{\omega}}^{t} \odot \mathbf{m}_{k}^{t},
    \end{align}
    where \(\mathcal{G}_{k}^{t}\) is neighborhood of client \(k\) with size \(|\mathcal{G}_{k}^{t}| = M\); all local models are sparse, i.e., \(\boldsymbol{\omega}_{j}^{t} = \boldsymbol{\omega}_{j}^{t} \odot \mathbf{m}_{j}^{t}\). 
\end{assumption}
%This assumption implies that the aggregation process is modulated by the element-wise multiplication of the average of the models with the mask of client \(k\), reflecting the personalized nature of the learning process in a federated environment.

\begin{proposition}\label{proposition1}
   Assume $\mathcal{K} = \{1, 2, \ldots, K\}$ clients,
   %with $M = |\mathcal{N}_{k}^{t}|$ neighbours at round $t$. 
   then, exactly $m$ neighbors in $\mathcal{N}_{k}^{t}$ have reuse index less than $k$ follows a hypergeometric distribution with \begin{equation}
        \mathbb{P}(m, k) = \mathbb{P}(|\mathcal{N}_{(a)k}| = m) = \frac{\binom{k-1}{m} \binom{K - k}{M - m}}{\binom{K - 1}{M}},
    \end{equation}  
    where $m<M$ and $|\mathcal{N}_{(a)k}|$ is subset of $\mathcal{N}_{k}^{t}$ with index less than $k$.
\end{proposition}

\begin{proof}
    See Appendix \ref{appendix:client_hypergeometric}
\end{proof}
$\Tilde{\boldsymbol{\omega}}_{k}^{t}$ denotes the local personalized aggregated model for the $k$-th client at time $t$; The global aggregated model at time $t$, $\Tilde{\boldsymbol{\omega}}^{t}$, is defined as the average of the local aggregated models, i.e., $\Tilde{\boldsymbol{\omega}}^{t} = \frac{1}{K}\sum_{k=1}^{K}\Tilde{\boldsymbol{\omega}}_{k}^{t}$, where $K$ is the total number of clients. Let $M$ represent the number of clients in the neighborhood. All models are under the setup of DA-DPFL, where (1) the models are sparse; and (2) a new scheduling strategy is adopted. Then, we obtain the following theorem.
% \begin{theorem}\label{theorem:1}
% Suppose Assumptions \ref{assumption:mu_lips}, \ref{assumption:bounded_g}, \ref{assumption:bounded_var}, \ref{assumption:mask_influence}, and Proposition \ref{proposition1} hold. Moreover, assume the stepsize $\eta$ for SGD for training client models satisfies $\eta \leq \sqrt{\frac{1}{3\mu^{2}(z-2)(z-1)}}$ for some $z>2$ and $M$ is the neighborhood size. As $T$ is sufficiently large,
% \begin{align}
%     &\min{\mathbb{E}\|\nabla f(\Tilde{\boldsymbol{\omega}}^{t})\|^{2}} \notag \\
%     &\leq \frac{2}{3z\eta^{2}T(z-2)(1-\mu) S_{1}} \left[ \mathbb{E}(f(\Tilde{\boldsymbol{\omega}}^{0})) - \min{f} \right] \notag \\
%     &- \frac{S_{2}}{3z\eta^{2}} + \frac{2S_{3}}{3z\eta^{2}(z-2)(1-\mu) S_{1}}
% \end{align}
% where constants $S_{1}=\left( \exp{\left(\frac{(3M+2)E_{l}}{(2M+2)(z-2)}\right)} - 1 \right)$, $S_{2}=\frac{z}{z-1}\eta^{2}\sigma^{2}_{l} + 3z\eta^{2}(\sigma^{2}_{g}+\sigma^{2}_{p})+ 3z\eta^{2}$
% and $S_{3}=\frac{(3M+2)E_{l}\eta^{2}(G^{2}+\sigma_{l}^{2})}{4(M+1)}$. 
% $f(\Tilde{\boldsymbol{\omega}}^{0})$ represents the initial global model loss and $\min{f}$ denotes the minimum of loss.
% \end{theorem}

\begin{theorem}\label{theorem:1}
Under Assumptions~\ref{assumption:mu_lips}~to~\ref{assumption:mask_influence}, when $T$ is sufficiently large and the stepsize $\eta$ for SGD for training client models satisfies $\eta \leq \sqrt{\frac{1}{12\mu^{2}(M-1)(2M-1)}}$ for $M>1$,
\begin{align}
    &\min{\mathbb{E}\|\nabla f(\Tilde{\boldsymbol{\omega}}^{t})\|^{2}} \notag \\
    &\leq \frac{2}{T\left(\eta-6S_{1}(\mu-\eta)\right)} \left( \mathbb{E}[f(\Tilde{\boldsymbol{\omega}}^{0})] - \min{f} \right) + S_{3},
\end{align}
where $S_{1}=2\eta^{2}M(M-1)\left( \exp{\left(\frac{(3M+2)E_{l}}{4(M^{2}-1)}\right)} - 1 \right)$, $S_{2}=\frac{1}{2M-1}\sigma^{2}_{l} + 3(2\sigma^{2}_{g}+\sigma^{2}_{p})$
, and $S_{3}=\frac{2}{\eta-6S_{1}(\mu-\eta)} \cdot \left[ (\mu-\eta) S_{1}S_{2} +\frac{3\mu^{2}\eta^{3}(3M+2)E_{l}}{2(M+1)M}(\sigma_{l}^{2}+\sigma_{g}^{2})\right]$. 
$f(\Tilde{\boldsymbol{\omega}}^{0})$ represents initial global model loss, $\min{f}$ is minimum of loss, $M$ is neighborhood size.
\end{theorem}

\begin{proof}
See Appendix \ref{appendix_theorem1}
\end{proof}
\begin{remark}
Theorem \ref{theorem:1} reveals that with sufficiently large $T$, the error due to initial model loss and bounded variance for gradients become negligible. Specifically, if one can choose $\eta=\mathcal{O}(\frac{1}{\mu\sqrt{T}})$, the convergence boundary will be dominated by the rate of $\mathcal{O}\left(\frac{1}{\sqrt{T}}+\frac{\sigma^{2}_{l}+\sigma^{2}_{g}+\sigma^{2}_{p}}{\sqrt{T}}+\frac{\sigma^{2}_{l}+\sigma^{2}_{g}}{T}\right)$.%This outcome is contingent upon the fixed error terms associated with $S_{1}$, $S_{2}$, and $S_{3}$. In particular, within $S_{2}$: $\sigma_{l}^{2}$ quantifies the discrepancy between the estimated and expected gradients, $\sigma_{g}^{2}$ measures the divergence between local and global gradients, reflecting the extent of heterogeneity, and $\sigma_{p}^{2}$ constrains the gradient variance between the personalized global model and the standard global model.
\end{remark}

\begin{remark}
Theorem \ref{theorem:1} is consistent with two key empirical observations: (i) The number of communication rounds required to attain a specified error level $\varepsilon$ is lower compared to the DisPFL model. \textbf{Note:} This efficiency gain is attributed to the term $S_{1}>\left( e^{\frac{E_{l}}{(2M-2)}} - 1 \right)$, i.e., no scheduling involved, which emerges from our scheduling strategy. The division of the left-hand side (first and third item) of the inequality by $S_{1}$ results in a reduced error boundary. (ii) Changed ratio is $\frac{3M+2}{2M+2}$. When $M=2$, the ratio  simplifies to $\frac{4}{3}$. As $M$ increases, this ratio approaches $\frac{3}{2}$. This indicates that while increasing $M$ enhances the error-bound reduction, the improvement rate diminishes, suggesting a limit to the benefits offered DA-DPFL scheduling.
\end{remark}

\begin{figure*}[ht!]
    \centering
    \includegraphics[width=\textwidth]{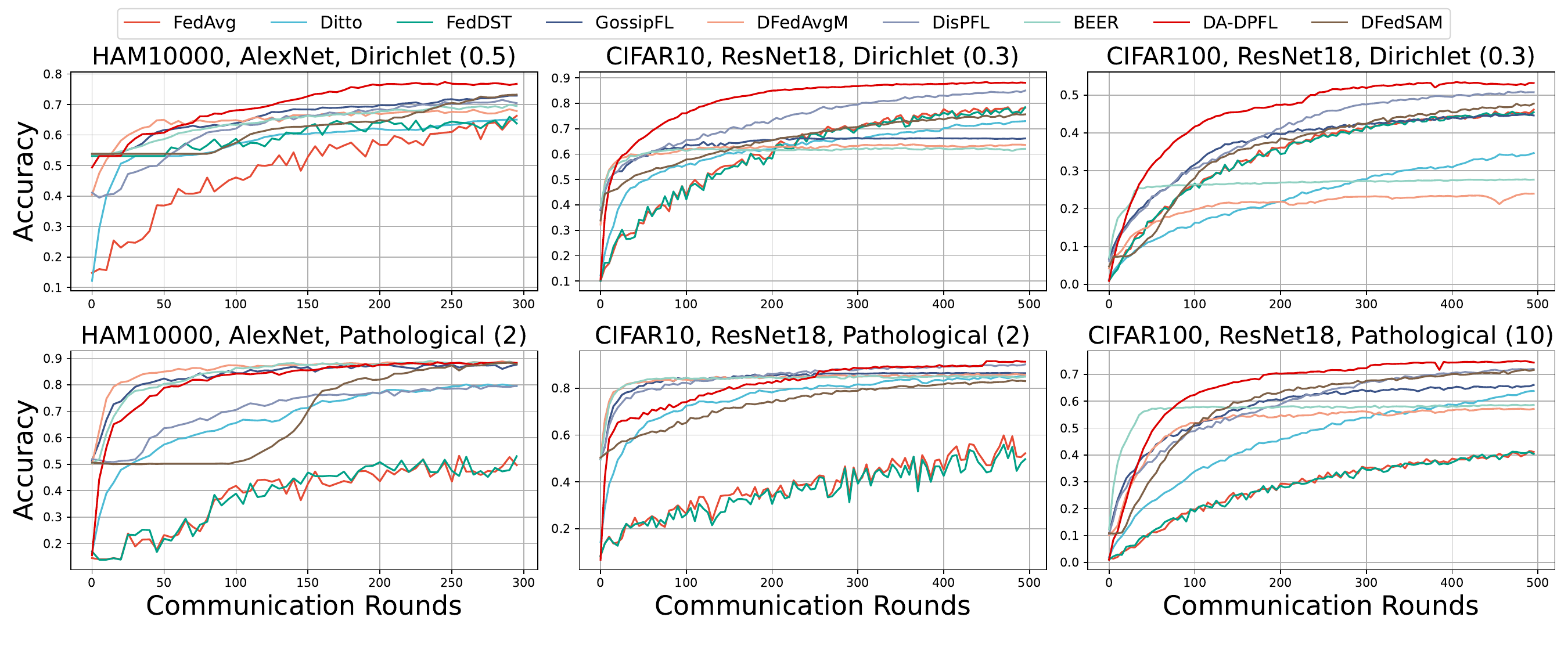}
    \caption{Test (\textit{top-1}) accuracy of all baselines, including CFLs and DFLs, across various model architectures and datasets.}
    \label{figure:baselines_comparison}
\end{figure*}
\section{Experiments}
\subsection{Experimental Setup}
\label{subsec:experiment}
%\subsubsection{Datasets \& Models}
\subsubsection{Datasets \& Models} Our experiments were conducted on three widely-used datasets: HAM10000 \cite{tschandl2018ham10000}, CIFAR10, and CIFAR100 \cite{krizhevsky2009learning}. We employed two distinct partition methods, \textbf{Pathological} and \textbf{Dirichlet}, to generate non-i.i.d. scenarios paralleling the approach in \cite{dai2022dispfl}. We use \textit{Dir} for Dirichlet and \textit{Pat} for Pathological in the following notations. The Dir. partition constructs non-i.i.d. data using a Dir($\alpha$) distribution, with $\alpha = 0.3$ for CIFAR10 and CIFAR100, and $\alpha = 0.5$ for HAM10000. For Pat. partitioning, several classes $n_{cls}$ are assigned per client: $2$ for CIFAR10 and HAM10000, and $10$ for CIFAR100.
To validate the versatility of our pruning methods across various model architectures, we selected AlexNet \cite{krizhevsky2012imagenet} for HAM10000, ResNet18 \cite{he2016deep} for CIFAR10, and VGG11 \cite{vggsimonyan2015very} for CIFAR100, ensuring a comprehensive evaluation across diverse model structures.

\subsubsection{Baselines} We compare the proposed methods with baselines including CFL: \textbf{FedAvg}\cite{McMahan2017}, \textbf{Ditto}\cite{Li2021a} and \textbf{FedDST} \cite{bibikar2022federated}, and DFL: \textbf{GossipFL}\cite{tang2022gossipfl}, \textbf{DFedAvgM}\cite{sun2022decentralized}, \textbf{DisPFL}\cite{dai2022dispfl}, \textbf{BEER}\cite{zhao2022beer} and \textbf{DFedSAM}\cite{shi2023improving}. 

\subsubsection{System Configuration} We consider a network of $K=100$ clients and select $M=N=10$ clients (neighbors) per communication round. CFL focuses on the communication between central server and selected clients. DFL mirrors this communication allocating identical bandwidth to each of the busiest clients. This ensures that $10$ clients are active per round matching server's connection load. 
All the baseline results are average values for three random seeds of best test model performance. In contrast to CFL, all DFL baselines except DFedSAM are configured with half the communication cost. Our method, FedDST, and DisPFL implement sparse model training for efficiency, which all start with initial sparsity $s_{k}^{0}=0.5$ with $k\in [K]$ for all clients. To ensure a balanced and fair comparison, all DFL benchmarks incorporate personalization by monitoring model performance following local training under randomly time-varying connection in \cite{dai2022dispfl}.

\subsubsection{Hyperparameters} To ensure a fair comparison, we align our experimental hyperparameters with the setups described in \cite{dai2022dispfl} and \cite{shi2023improving}.  Unless otherwise specified, we fix the number of local epochs at $5$ for all approaches and employ a Stochastic Gradient Descent (SGD) optimizer with a weight decay set to $5 \times 10^{-4}$. The learning rate is initialized at $0.1$, undergoing an exponential decay with a factor of $0.998$ after each global communication round. The batch size is consistently set to $128$ across all experiments. The global communication rounds are conducted 500 times for the CIFAR10 and CIFAR100 datasets, and 300 times for the HAM10000 dataset. We let $\delta_{v}=0.5,b=0,c=1.3$, $\{p,q,\gamma,\eta_{c}\}=\{0.5,1,0.9,1\}$ as suggested by \cite{diao2022pruning}, and  $\delta_{pr} \in \{0.01, 0.02,0.03\}$ for all experiments. In the CFL baseline implementation, the local training for Ditto is bifurcated into two distinct phases: a global model training phase spanning 3 epochs and a personalized model training phase consisting of 2 epochs. Additionally, the update mask reconfiguration interval in FedDST is determined through a grid search within the set \([1, 5, 10, 20]\). In our DFL setup, when algorithms incorporate compression techniques, we manage to reduce half of the busiest communication load. In contrast, GossipFL utilizes a \textit{Random Match} approach, which entails randomly clustering clients into specific groups. For the optimization algorithm, Momentum SGD is adopted in DFedAvgM and DFedSAM, with a momentum factor \(\beta = 0.9\). Additionally, a \(\rho\) value for DFedSAM is decided by grid search from $[0.01, 0.02, 0.05, 0.1, 0.2, 0.5]$, following the work for SAM optimizer.

\begin{figure}[htbp]
    \begin{center}
        \includegraphics[width=0.35\textwidth]{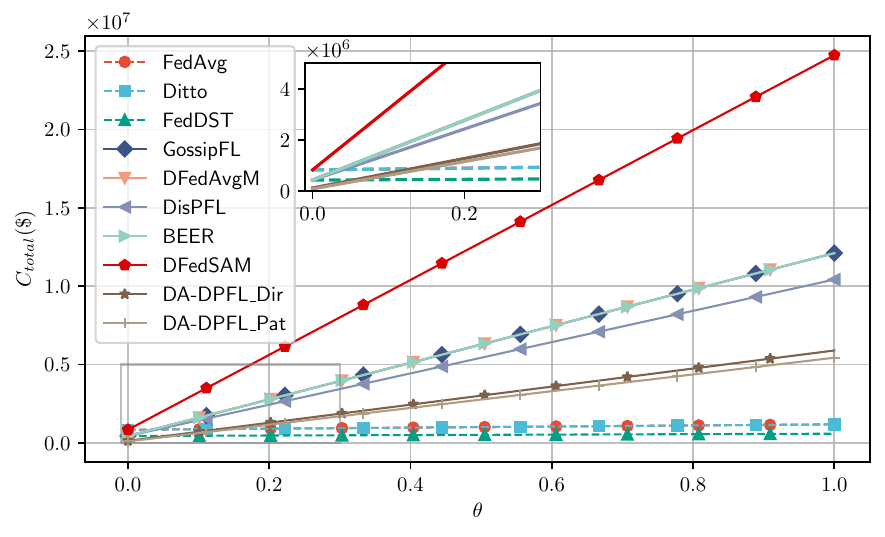}    
    \caption{Total cost (energy and time cost, in USD) of DA-DPFL and all baselines evaluated on CIFAR10 against $\theta$.}
    \label{fig:total_cost}
    \end{center}
\end{figure}

\begin{figure}[htbp]
    \begin{center}
        \includegraphics[width=0.35\textwidth]{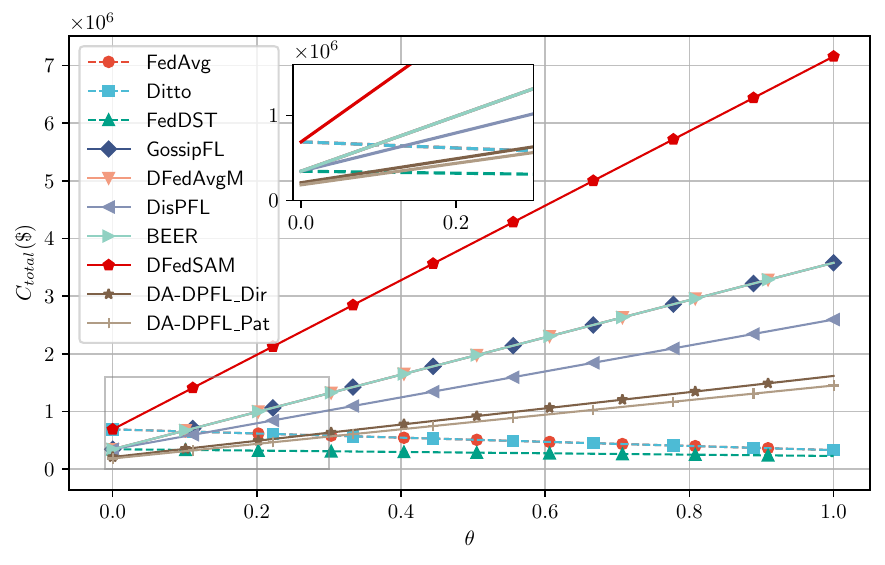}    
    \caption{Total cost (energy and time cost, in USD) of DA-DPFL and all baselines evaluated on CIFAR100 against $\theta$.}
    \label{fig:total_cost_cifar100}
    \end{center}
\end{figure}

% \begin{figure}[htbp]
%     \begin{center}
%         %\includegraphics[trim=left bottom right top, clip]{networkschedule.pdf}
%         \includegraphics[width=0.35\textwidth]{Total_cost_HAM10000.pdf}    
%     \caption{Total cost (energy and time cost, in USD) of DA-DPFL and all baselines evaluated on HAM10000 against $\theta$.}
%     \label{fig:total_cost_HAM10000}
%     \end{center}
% \end{figure}

\subsubsection{Cost Simulation}
\label{appendix:C2}
For cost analysis, we utilized an NVIDIA 4090 GPU with 80 TFLOPS and 450 W TDP as a standard to assess clients' computational power and energy consumption. The architecture anticipates 1 Gbps bandwidth, with client network cards utilizing 1 W, cited from \cite{feeney2001investigating}. We derived $C_{\text{time}}$ from Table \ref{table:comms_flops_comparison} and converted $C_{\text{energy}}$ and $C_{\text{time}}$ to monetary units using $(1-\theta)\$/s$ and $\theta \$/J$, illustrated in Fig. \ref{fig:total_cost}. To address the gap between theoretical and actual GPU execution times, we performed real-world algorithm executions on the GPU. These revealed a fivefold increase over theoretical times, leading to a correction factor of 5 for computation time, calculated as $T_{\text{comp}} = 5 \times \frac{D_{\text{FLOP}}}{V_{\text{FLOPS}}}$, and energy, $C_{\text{comp}} = T_{\text{comp}} \times P_{\text{comp}}$. Similarly, in estimating the communication time $T_{\text{comm}}$, we apply the formula $T_{\text{comm}} = \frac{D_{\text{comm}}}{B}$, where $D_{\text{comm}}$ denotes the data volume to be transferred and $B$ signifies the system's overall bandwidth. Accordingly, the communication energy cost $C_{\text{comm}}$ is determined by $C_{\text{comm}} = T_{\text{comm}} \times P_{\text{comm}}$, with $P_{\text{comm}}$ indicating the transmission power of the wireless network card.

\subsection{Performance Analysis}

\subsubsection{Test Accuracy Evaluation} DA-DPFL outshines all baselines in \textit{top}-1 accuracy across five out of six scenarios, maintaining robustness under extreme non-iid conditions ($n_{cls}=2$) (Fig. \ref{figure:baselines_comparison}, Table \ref{table:accuracy_comparison}). It exceeds the next best DFL baselines (DisPFL and GossipFL) by $2-3\%$, with a minor shortfall in HAM10000 ($n_{cls}=2$) by $0.5\%$ against DFedAvgM. DA-DPFL consistently surpasses DisPFL in sparse model training and generalization, while maintaining efficient convergence. Conversely, CFL lags in convergence due to its limited client participation per round. Momentum-based methods like DFedAvgM show accelerated initial learning, while BEER, with gradient tracking, exhibits rapid convergence but does not necessarily reduce generalization error. DA-DPFL demonstrates a balanced trade-off between convergence rate and generalization performance, outperforming other baselines achieving target accuracy with reduced costs.

\subsubsection{Efficiency} We evaluate the efficiency of our algorithm by analyzing two key metrics: the Floating Point Operations (FLOP) required for inference, and communication overhead incurred during convergence rounds. To ensure a fair comparison, we employ the initialization protocol from DisPFL, thereby standardizing the initial communication costs and FLOP values in the initial pruning phase of training. 
Notably, the pruning stages integral to DA-DPFL lead to a significant reduction in these costs. This is evidenced by the final sparsity levels achieved: $(0.61, 0.56)$ for HAM10000, $(0.65, 0.73)$ for CIFAR10, and $(0.70, 0.73)$ for CIFAR100 under Dir. and Pat. partitioning, respectively. These results are obtained within the constrained communication rounds. A critical observation is that both the busiest communication costs and training FLOPs for our approach are lower compared to the most efficient DFL baseline, DisPFL. These comparative insights are further elaborated in Table \ref{table:comms_flops_comparison}, with bold values underscoring the efficiency of DA-DPFL.
To quantify the impact of a potential delay in DA-DPFL, we adopt metrics to calculate the total cost $C_{\text{total}}$ defined in \cite{luo2021cost} and  \cite{zhou2022joint} as:
%\begin{equation}\label{eq:energy_time_cost}
$C_{\text{total}} = (1-\theta)C_{\text{time}} + \theta C_{\text{energy}}$,
%\end{equation}
where $\theta \in [0,1]$ is set to 0 for extreme time-sensitive applications and to 1 for extreme energy-sensitive tasks. This metric allows for a unified representation of time and energy costs in monetary units (USD $\$$). 
% Due to the notable scarcity of research focusing on the quantification of time costs in units comparable to FLOPs or communication overheads, to showcase DA-DPFL's edge in an intuitive way, we propose a conservative approach where the cost of time is equated to FLOPs or communication overheads, representing a worst-case scenario in DA-DPFL. In contrast, for other baseline methods, this cost metric is reduced by a factor of 17, as elucidated in Fig. \ref{fig:waiting_delay}, ensuring a balanced comparison.
To provide a realistic and practical insight into how the introduction of DA-DPFL would affect the total cost needed for the whole process of FL, we chose to combine the communication and the computation cost (FLOP) in the form of energy expenditure, i.e.,
%\begin{equation}\label{eq:energy_cost}
$C_{\text{energy}} = C_{\text{comm}} + C_{\text{comp}}$,
%\end{equation}
where $C_{\text{comm}}$ and $C_{\text{comp}}$ is communication and computational cost, respectively. %Details about the simulation are included in  \ref{sec:cost_analysis}. 
Figures \ref{fig:total_cost} and \ref{fig:total_cost_cifar100} show the cost-effectiveness of DA-DPFL compared to other DFL baselines. Initially, when $\theta \to 0$ %(emphasizing the primacy of time over computation cost), 
, DA-DPFL incurs a higher time cost. However, as $\theta$ increases beyond 0.2, DA-DPFL demonstrates remarkable advantages over the other DFL algorithms (represented by solid lines), with its lead expanding as $\theta$ further increases. Due to the system configuration, CFLs have significantly lower communication (1\%) and computation (10\%) costs compared to DFL, which (indicated by dotted lines) exhibits superior cost efficiency overall, but lower convergence speed. Overall, even when considering waiting time, DA-DPFL successfully achieves both cost and learning (convergence speed) efficiency.

%Note that we do not apply any optimization to (\ref{eq:energe_time_cost}); instead, the inherent decisions of the algorithms determine the total cost. 
%DA-DPFL offers substantial benefits in terms of energy being 
%well-suited for energy-critical applications.

\begin{table*}[ht]
\centering
\caption{Accuracy comparison of federated learning methods across different datasets}
\begin{tabular}{lcccccc}
\hline
& \multicolumn{2}{c}{HAM10000} & \multicolumn{2}{c}{CIFAR10} & \multicolumn{2}{c}{CIFAR100} \\ 
\hline %\cline{2-7}
&  Dir. $(0.5)$ & Pat. $(2)$ & Dir. $(0.3)$ & Pat. $(2)$ & Dir. $(0.3)$ & Pat. $(10)$ \\
\hline
FedAvg \cite{McMahan2017} & 65.92 $\pm0.3$ & 55.68 $\pm0.4$ & 79.30 $\pm0.2$ & 60.09 $\pm0.2$ & 46.21 $\pm0.4$ & 41.26 $\pm0.3$ \\
Ditto \cite{Li2021a} & 65.19 $\pm0.2$ & 80.17 $\pm0.1$ & 73.21 $\pm0.2$ & 85.78 $\pm0.1$ & 34.83 $\pm0.2$ & 64.41 $\pm0.3$ \\
%pFedGate \cite{chen2023efficient} & - & - & - & - & - & - \\
FedDST \cite{bibikar2022federated} &66.11 $\pm 0.3$ & 55.07 $\pm 0.4$ & 78.47 $\pm 0.2$ & 56.32 $\pm 0.3$ & 46.01 $\pm0.2$ & 41.42 $\pm 0.2$ \\
\hline
GossipFL \cite{tang2022gossipfl} & 72.92 $\pm 0.1$ & 88.05 $\pm 0.1$ & 66.43 $\pm 0.1$ & 86.60 $\pm0.1$ & 45.09 $\pm0.1$ & 66.03 $\pm0.1$ \\
DFedAvgM \cite{sun2022decentralized} & 68.30 $\pm 0.1$ & \textbf{88.89} $\pm 0.1$ & 65.05 $\pm 0.1$ & 85.34 $\pm0.2$ & 24.11 $\pm 0.1$ & 57.41 $\pm 0.1$ \\
DisPFL \cite{dai2022dispfl} & 71.56 $\pm 0.1$ & 80.09 $\pm 0.1$ & 85.85 $\pm 0.2$ & 90.45 $\pm 0.2$ & 51.05 $\pm0.3$ & 72.22 $\pm0.2$ \\
BEER \cite{zhao2022beer} & 69.80 $\pm0.1$ & 88.75 $\pm0.2$ & 62.94 $\pm 0.1$ & 85.48 $\pm0.1$ & 27.79 $\pm0.1$ & 58.71 $\pm0.1$ \\
DFedSAM \cite{shi2023improving} & 73.74 $\pm 0.2$ & 88.47 $\pm 0.3$ & 75.74 $\pm 0.2$ & 83.51 $\pm0.1$ & 47.86 $\pm0.2$ & 71.76 $\pm 0.1$ \\
DA-DPFL (\textit{Ours}) & \textbf{76.32} $\pm0.3$ & 88.36 $\pm0.3$ & \textbf{89.08 $\pm0.3$} & \textbf{91.87 $\pm0.1$} & \textbf{53.53 $\pm0.2$} & \textbf{74.91 $\pm0.1$} \\
\hline
\end{tabular}
\label{table:accuracy_comparison}
\end{table*}

% \begin{table}[ht]
% \centering
% \caption{Busiest Communication Cost and Final Training FLOPs of federated learning methods across different datasets}
% \begin{tabular}{|p{1.9cm}|c|c|c|c|c|c|}
% \hline
% & \multicolumn{2}{c|}{HAM10000} & \multicolumn{2}{c|}{CIFAR10} & \multicolumn{2}{c|}{CIFAR100} \\ 
% \cline{2-7}
% & Com. & FLOP & Com. & FLOP & Com. & FLOP \\

% & (MB) & (1e12) & (MB) & (1e12) & (MB) & (1e12) \\
% \hline
% FedAvg \cite{McMahan2017} & 887.8 & 3.6 &  426.3 & 8.3 & 353.3 & 2.3 \\
% Ditto \cite{Li2021a} & 887.8 & 3.6 &  426.3 & 8.3 & 353.3 & 2.3\\
% FedDST \cite{bibikar2022federated} & 443.8 & 2.0 & 223.1 & 4.1 & 176.7 & 1.6 \\
% %pFedGate \cite{chen2023efficient} & - & - & - & - & - & - \\
% GossipFL \cite{tang2022gossipfl} & 443.8 & 3.6 & 223.1 & 8.3 & 176.7 & 2.3 \\
% DFedAvgM \cite{sun2022decentralized} & 443.8 & 3.6 & 223.1 & 8.3 & 176.7 & 2.3 \\
% DisPFL \cite{dai2022dispfl} & 443.8 & 2.0 & 223.1 & 7.1 & 176.7 & 1.6 \\
% BEER \cite{zhao2022beer} & 443.8 & 3.6 & 223.1 & 8.3 & 176.7 & 2.3 \\
% DFedSAM \cite{shi2023improving} & 887.8 & 7.2 & 426.3 & 17 & 353.3 & 4.6 \\
% DA-DPFL\_Dir & \textbf{346.2} & \textbf{1.9} & \textbf{149.1} & \textbf{4.1} & \textbf{107.7} & \textbf{1.0} \\
% DA-DPFL\_Pat & \textbf{394.4} & \textbf{2.1} & \textbf{115.1} & \textbf{3.8} & \textbf{94.8} & \textbf{0.9} \\
% \hline
% \end{tabular}
% \label{table:comms_flops_comparison}
% \end{table}

\begin{table}[tb]
\centering
\caption{Busiest Communication Cost \& Final Training FLOPs of all methods.}
\begin{tabular}{p{1.7cm}@{\hspace{0.2em}}c@{\hspace{0.2em}}@{\hspace{0.2em}}c@{\hspace{0.2em}}@{\hspace{0.2em}}c@{\hspace{0.2em}}@{\hspace{0.2em}}c@{\hspace{0.2em}}@{\hspace{0.2em}}c@{\hspace{0.2em}}@{\hspace{0.2em}}c@{\hspace{0.2em}}}
\hline
& \multicolumn{2}{@{\hspace{0.2em}}c}{HAM10000} & \multicolumn{2}{c}{CIFAR10} & \multicolumn{2}{c}{CIFAR100} \\ 
\hline %\cline{2-7}
& Com. & FLOP & Com. & FLOP & Com. & FLOP \\
& (MB) & (1e12) & (MB) & (1e12) & (MB) & (1e12) \\
\hline
FedAvg & 887.8 & 3.6 &  426.3 & 8.3 & 353.3 & 2.3 \\
Ditto & 887.8 & 3.6 &  426.3 & 8.3 & 353.3 & 2.3\\
FedDST & 443.8 & 2.0 & 223.1 & 7.1 & 176.7 & 1.6 \\
\hline
GossipFL & 443.8 & 3.6 & 223.1 & 8.3 & 176.7 & 2.3 \\
DFedAvgM & 443.8 & 3.6 & 223.1 & 8.3 & 176.7 & 2.3 \\
DisPFL & 443.8 & 2.0 & 223.1 & 7.1 & 176.7 & 1.6 \\
BEER & 443.8 & 3.6 & 223.1 & 8.3 & 176.7 & 2.3 \\
DFedSAM & 887.8 & 7.2 & 426.3 & 17 & 353.3 & 4.6 \\
DA-DPFL\_Dir & \textbf{346.2} & \textbf{1.9} & \textbf{149.1} & \textbf{4.1} & \textbf{107.7} & \textbf{1.0} \\
DA-DPFL\_Pat & \textbf{394.4} & \textbf{2.0} & \textbf{115.1} & \textbf{3.8} & \textbf{94.8} & \textbf{0.9} \\
\hline
\end{tabular}
\label{table:comms_flops_comparison}
\end{table}

%\begin{figure}[htbp]
%    \centering
%    \includegraphics[width=0.48\textwidth]%{Num_of_neighbors.pdf}
%    \caption{Impact of the number of neighbors involved in each training round on accuracy (CIFAR10, \textit{Dir}(0.3), $\delta_{pr}=0.03$)}
%    \label{fig:num_neighbors}
%\end{figure}

\subsubsection{Extended Topology}
To demonstrate the adaptability of our DA-DPFL, we conducted further experiments utilizing both \textit{ring} and \textit{fully-connected} (FC) topologies. 
These experiments were carried out in comparison with the above baselines using a \textit{Dirichlet} partition with $\alpha=0.3$. The results, presented in Table \ref{tab:ring_fc_results} show that DA-DPFL consistently surpasses other baselines, achieving higher performance with sparser models, within $500$ communication rounds. 
DA-DPFL maintains a significant lead in performance.

\begin{table}[ht]
\centering
\caption{Performance comparison for ring and fully connected topologies}
\label{tab:ring_fc_results}
\begin{tabular}{llcc}
\hline
Topology & Method    & Acc (\%) & Sparsity (s) \\ \hline
Ring     & GossipFL  & 66.12 $\pm 0.1$    & 0.00 \\
         & DFedAvgM  & 65.89 $\pm 0.1$    & 0.00 \\
         & DisPFL    & 67.65 $\pm 0.2$    & 0.50 \\
         & BEER      & 62.92 $\pm 0.1$    & 0.00 \\
         & DFedSAM   & 66.61 $\pm 0.2$    & 0.00 \\
         & DA-DPFL   & \textbf{69.83} $\pm 0.3$ & 0.65 \\ \hline
FC       & GossipFL  & 71.22 $\pm 0.2$    & 0.00 \\
         & DFedAvgM  & 69.89 $\pm 0.1$    & 0.00 \\
         & DisPFL    & 86.54 $\pm 0.2$    & 0.50 \\
         & BEER      & 68.77 $\pm 0.1$    & 0.00 \\
         & DFedSAM   & 79.63 $\pm 0.3$    & 0.00 \\
         & DA-DPFL   & \textbf{89.11} $\pm 0.2$ & 0.68 \\ \hline
\end{tabular}
\end{table}

\subsection{Analysis on neighborhood size \(M\)} We train ResNet18 on CIFAR10 to examine the impacts of the hyper-parameters \(M\). The neighborhood size parameter \(M\) markedly influences the scheduling efficiency in DA-DPFL. A higher \(M\) value accelerates the convergence of our approach, primarily by enhancing the reuse of the trained model throughout the training process, albeit at the risk of potential delay.
As illustrated in Fig.\ref{fig:combined_plot} Bottom, while \(M=20\) slightly outperforms \(M=10\), it incurs approximately double the time delays. 
\begin{figure}[hptb]
    \centering
    \includegraphics[width=0.48\textwidth]{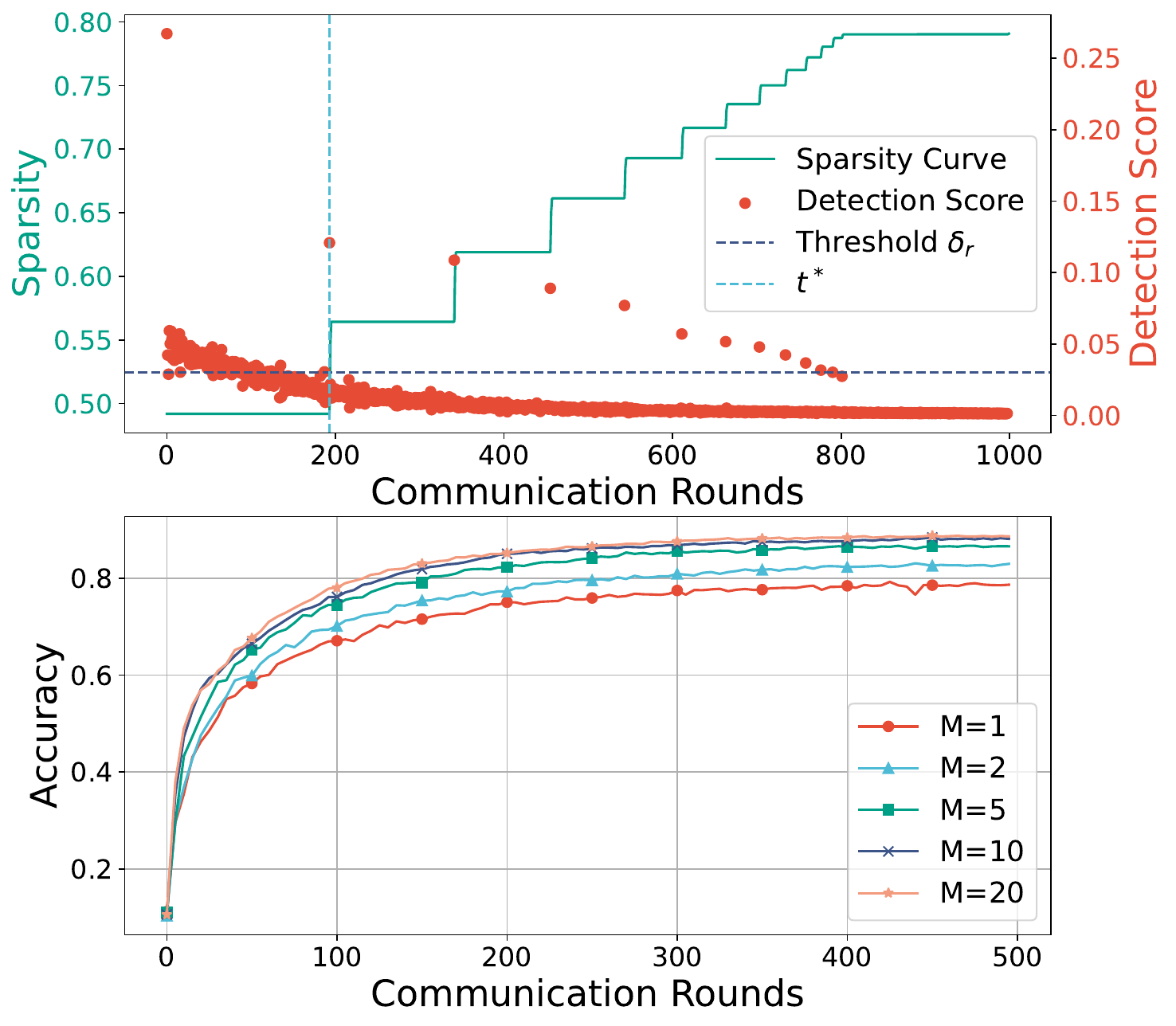}
    \caption{\textbf{(Top)} Relationship between sparsity and detection score; \textbf{(Bottom)} Impact of $M$ involved in each training round on accuracy (CIFAR10, \textit{Dir}(0.3), $\delta_{pr}=0.03$).}
    \label{fig:combined_plot}
\end{figure}
\subsection{Ablation Study}
\subsubsection{Threshold $\delta_{pr}$}
Extending total communication rounds from 500 to 1000, we ascertain that a target sparsity of \(s=0.8\) is attainable without compromising accuracy (DA-DPFL achieves 89\% over DisPFL’s 83.27\%). This finding challenges the generalization gap assumption in \cite{dai2022dispfl}, reducing the need for precise initial sparsity ratio selection in fixed sparsity pruning as DA-DPFL achieves equivalent or lower generalization error at higher sparsity levels through further pruning. Fig.\ref{fig:combined_plot}(Top) shows the pruning decisions based on average detection scores across clients and their sparsity trajectories. The initial high detection score validates the substantial disparity between the random mask and the RigL algorithm-derived mask, differing from EarlyCrop’s centralized, densely initialized model approach. Post \(t^{*}\), client models undergo incremental pruning in DA-DPFL, with the pruning scale diminishing due to reduced model compressibility, as evidenced by sparsity alterations at each pruning phase. To ascertain the effect of the early pruning threshold $\delta_{pr}$, we conducted experiments with CIFAR10 and ResNet18. Fig.\ref{fig:epth} underscores pruning timing significance, indicating varying optimal thresholds for different data partitions and corresponding detection score divergences. Early pruning, though accelerating sparsity achievement, impedes critical learning phases, while excessively delayed pruning equates to post-training pruning, incurring higher costs. Consequently, our results advocate for early-stage further pruning, ideally between 30-40\% of total communication rounds, aligning with a threshold range of 0.02-0.03, to balance model performance with energy efficiency.
\begin{figure}[htbp]
    \centering
    \includegraphics[width=0.48\textwidth]{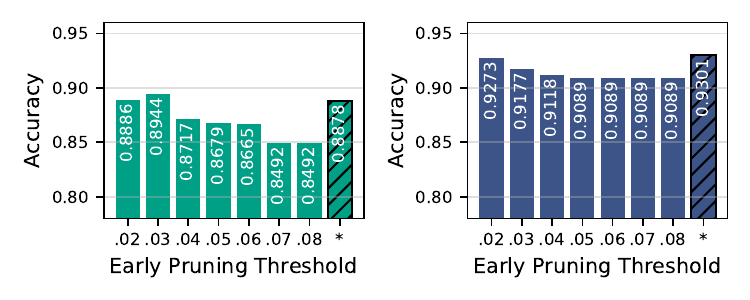}
    \caption{Impact of $\delta_{pr}$ on final prediction accuracy of achieving sparsity $s=0.8$ with CIFAR10 ($M=10$) \textit{Dir} (left) and \textit{Pat} (right) partitions (in which * stands for DA-DPFL without further pruning, i.e., fixed sparsity $s=0.5$).}
    \label{fig:epth}
\end{figure}
%\section{Limitations}
%To mitigate the potential delays incurred by waiting, DA-DPFL introduces two primary strategies. 
%Firstly, DA-DPFL establishes a maximum timeout for the waiting period. 
%Client \( k \) can only wait for \( t_{W} \) time units, i.e., interval between model aggregation, within the set \( \mathcal{N}_{(a)k}^{t*} \subseteq \mathcal{N}_{(a)k}^{t} \) instead of waiting for the slowest client in $ \mathcal{N}_{(a)k}^{t}$. This effectively curtails the delay. Note, \( \mathcal{N}_{(a)k}^{t*} \) comprises only those neighbors that are available within the timeout range \( t_{W} \). Secondly, we introduce a sequential line allowing clients to use 
%trained models for further training to maintain computational parallelism. 
%Multiple clients commence computations simultaneously from the 
%start of each round $t$. Task allocation across clients under different client indexing further aids this. 
%For instance, assume two different DL training tasks $T_{A}$ and $T_{B}$ and client $c_{a}$ indexed with $k=1$ and $k=10$ for tasks $T_{A}$ and $T_{B}$, respectively, while, conversely, client $c_{b}$ is indexed $k=10$ and $k=1$ for tasks $T_{A}$ and $T_{B}$, respectively. Then, $c_{a}$ and $c_{b}$ commence in parallel their tasks $T_{A}$ and $T_{B}$, respectively, without any waiting.
\subsubsection{Waiting Threshold $N$}
To ensure a fair comparison, we add $N=\{0,2,5\}$ with the same experiment setup as in \ref{subsec:experiment} for \textit{Dir} partition.
The experimental results in Fig. \ref{fig:wait_number} demonstrate a clear trend: increasing the \(N\) consistently improves model accuracy across the CIFAR10 and CIFAR100 datasets, with CIFAR10 seeing up to a 1.87\% increase and CIFAR100 a 1.41\% increase in accuracy from \(N=0\) to \(N=10\). Interestingly, the HAM10000 dataset shows no  \(N=5\) achieves the best performance, suggesting task-specific characteristics influence the optimal selection of $N$. Even the model performance for $N=0$ cases are higher than DisPFL, which illustrates the effectiveness of our \textit{further pruning} strategy. Furthermore, it is possible to obtain redundancy in reusing models, especially when $M$ is large. By selecting $N$, one can trade off between waiting and model performance.

\begin{figure}[htbp]
\centering
\includegraphics[width=2in]{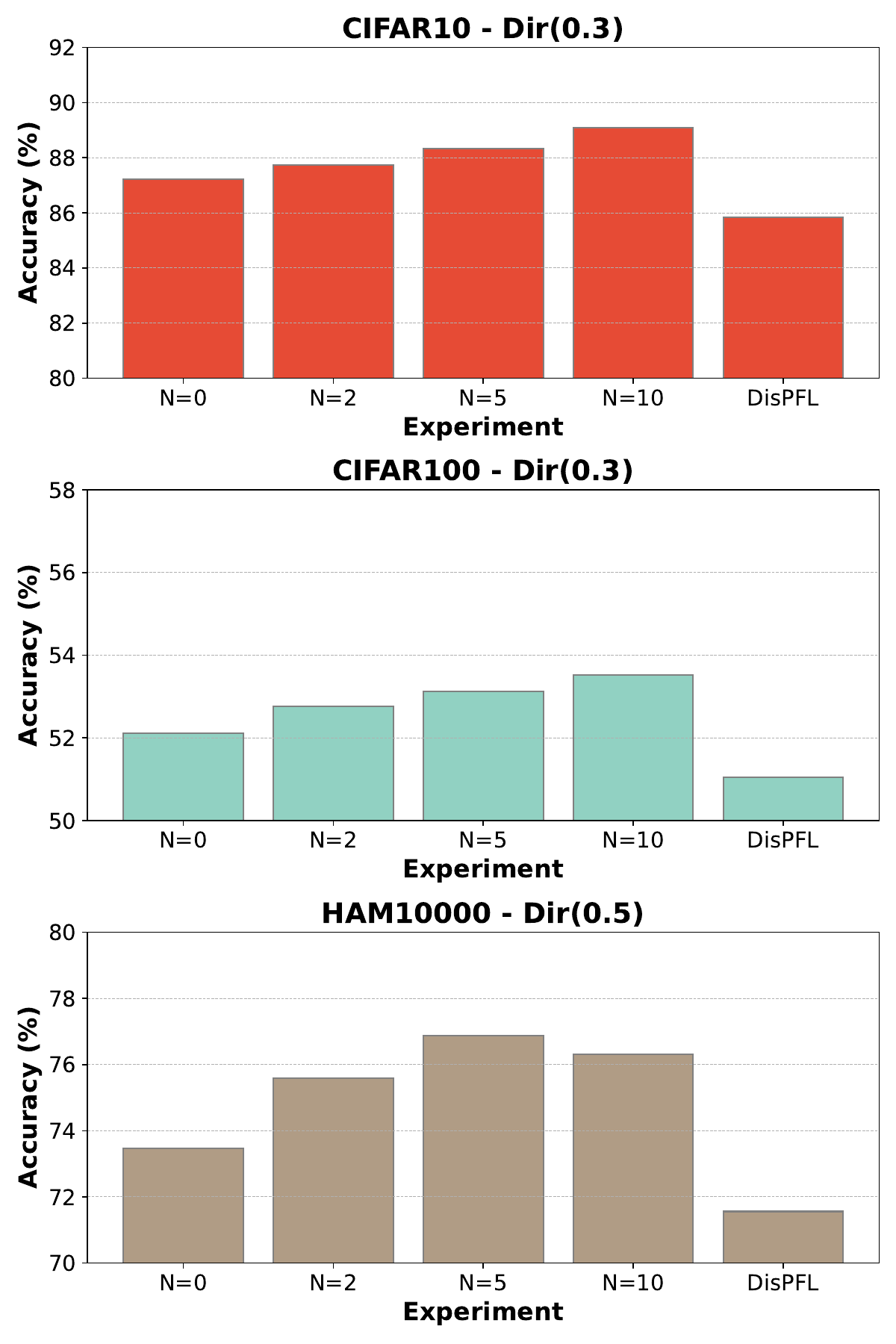}
\caption{Performance of different number of maximum waiting numbers $N$}
\label{fig:wait_number}
\end{figure}

\subsection{Parallelism and Delay}\label{sec:cost_analysis}
We conducted $10,000$ iterations to estimate the average impact on 
parallelism and latency attributed to waiting times. Here, we define parallelism as the proportion of clients that commence training concurrently. The result is depicted in Figures 
Figure \ref{fig:parallelism} illustrates a decline in parallelism as the number of clients in the neighborhood $M$ increases (having $K=100$ clients). 
Figure \ref{fig:waiting_delay} shows delay against $M$. The black line, representing $N=M$, delineates the outcome of awaiting the most delayed clients, i.e., \textit{without any control}. It is evidenced to scale almost linearly with the neighborhood size $M$. 
Moreover, in Figure \ref{fig:waiting_delay}, one can observe the efficacy of the constraint $N \leq M$ in mitigating delays while increasing $M$. %Given our simulated environment with uniform computational and communication capabilities, 
The mean maximum waiting time is indicative of the multiplier effect on the time required for each communication round relative to traditional decentralized FL. The lines corresponding to $N \in \{1,2,5,10,20\}$ corroborate that the waiting period can be effectively regulated by $N$. 
In cases where $M=N=100$, DA-DPFL transits to sequential learning, while with $M=100$ and $N=2$, DA-DPFL sustains a comparatively high degree of parallelism with opportunities for model reuse. 
With $N=0$, DA-DPFL reduces to DisPFL with our pruning strategy.

% \begin{figure}[!t]
% \centering
% \begin{subfigure}{2.5in}
%     \centering
%     \includegraphics[width=\linewidth]{parallelism.pdf}
%     \caption{Impact on Parallelism with Different Number of Neighbor Clients}
%     \label{fig:parallelism}
% \end{subfigure}
% \hfil
% \begin{subfigure}{2.5in}
%     \centering
%     \includegraphics[width=\linewidth]{waiting_time.pdf}
%     \caption{Delay Induced by Waiting with Different Number of Neighbor Clients}
%     \label{fig:waiting_delay}
% \end{subfigure}
% \caption{Characteristic of Proposed Time-varying Connected Topology. (a) Impact on Parallelism with Different Number of Neighbor Clients. (b) Delay Induced by Waiting with Different Number of Neighbor Clients.}
% \label{fig:analysis_on_topology}
% \end{figure}

\begin{figure}[htbp]
\centering
\includegraphics[width=2.5in]{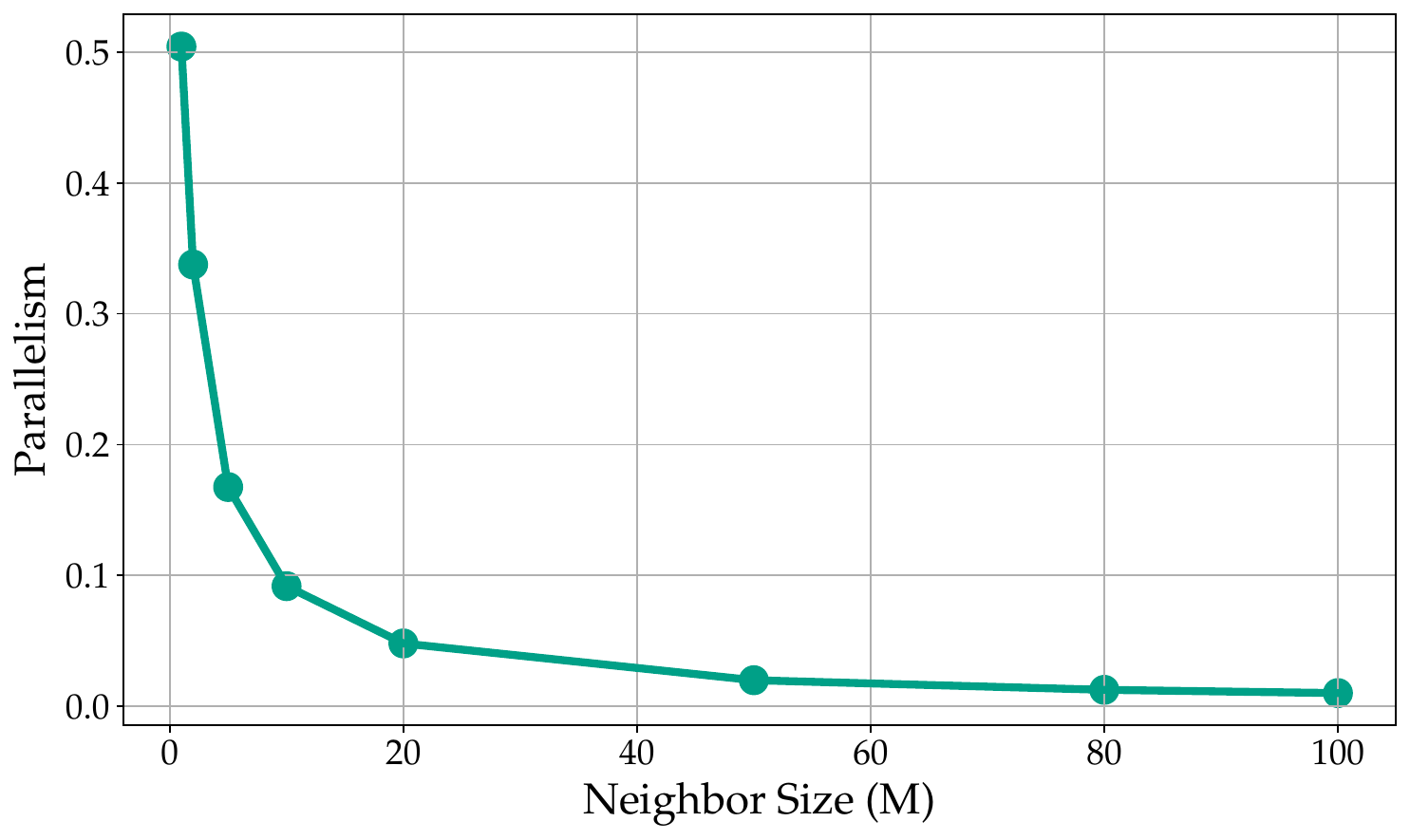}
\caption{Characteristic of proposed time-varying connected topology: impact on parallelism with different number of neighbor clients. }
\label{fig:parallelism}
\end{figure}

\begin{figure}[htbp]
\centering
\includegraphics[width=2.5in]{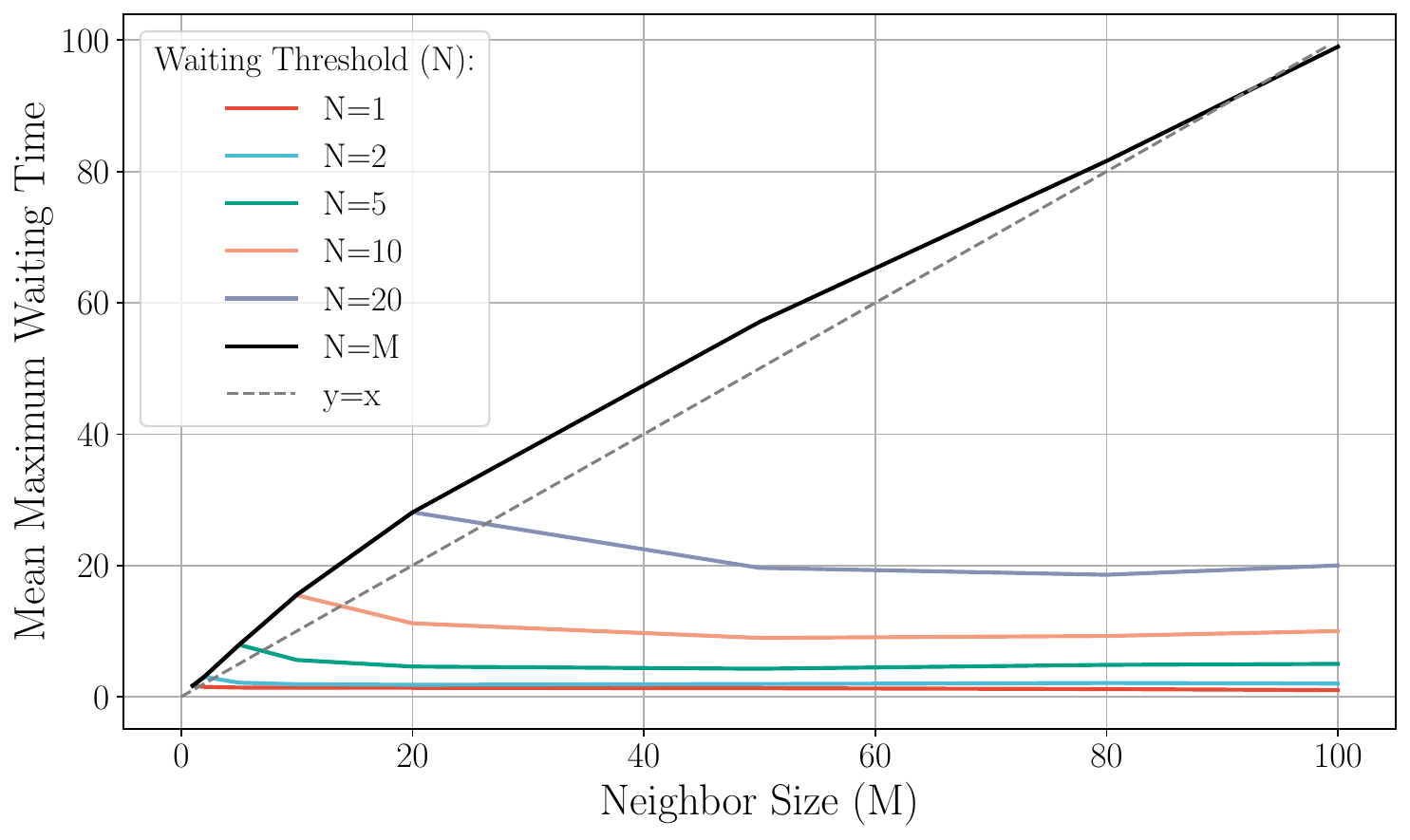}
\caption{Characteristic of proposed time-varying connected topology: delay induced by waiting with different number of neighbor clients.}
\label{fig:waiting_delay}
\end{figure}

\section{Conclusions}
DA-DPFL is a fair learning scheduling framework that cost-effectively deals with data heterogeneity. 
DA-DPFL conserves computational \& communication resources and accelerates the learning process by introducing a novel sparsity-driven pruning technique. 
%the trained model on other clients within the same communication round, 
We provide a theoretical analysis on DA-DPFL's convergence. Comprehensive experiments and comparisons with DFL and CFL baselines in PFL context showcase learning efficiency, enhanced model accuracy, and energy efficiency, which confirms the effectiveness of DA-DPFL in practical applications.
DA-DPFL sets the stage for future plans in adaptive algorithms handling time-series and graph data across diverse topologies expanding our models applicability to real-world scenarios.

\appendix

\section*{Additional Algorithms}
\subsection{SAP (PQI) Algorithm}\label{sec:appendix_SAP}
In our approach, we adopt the SAP algorithm, as shown in Algorithm \ref{algorithm:SAP}, to assess the compressibility of neural networks, characterized by four distinct features. Firstly, the initial model employed in our study is inherently sparse. Secondly, we implement PQI pruning as a further pruning technique within a Federated Learning (FL) framework, based on other fixed pruning methodologies. Thirdly, our method incorporates a meticulously designed pruning strategy that ensures proper pruning frequency and specifically avoids further pruning during the \textit{critical learning} period. Lastly, unlike conventional practices, we integrate the SAP algorithm during the training phase, as opposed to applying it post-training.

\begin{algorithm}
\caption{PQI-driven pruning (Layerwise)}
\label{algorithm:SAP}
\begin{algorithmic}[1]
\STATE \textbf{Input:} $\tilde{\boldsymbol{\omega}}_{k,E_{l}}^{t}$, mask $\mathbf{m}_{k}^{t}$, norm index $0 < p \leq 1 < q$, compression hyper-parameter $\eta_{c}$, scaling factor $\gamma$, pruning threshold $\beta$, further pruning time $\mathcal{T}$.
\STATE \textbf{Output:} $\tilde{\boldsymbol{\omega}}_{k,E_{l}}^{t\prime}$, corresponding mask $\mathbf{m}_{k}^{t\prime}$
\FOR{$t \in \mathcal{T}$}
    \FOR{each layer $l \in |L|$}
        \STATE Compute dimensionality of $\tilde{\boldsymbol{\omega}}_{k,E_{l}}^{l,t}$: $d_{t}^{l} = |\mathbf{m}_{k}^{l,t}|$
        \STATE Compute PQ Index $I(\tilde{\boldsymbol{\omega}}_{k,E_{l}}^{l,t}) = 1 - \left(\frac{1}{d_{t}^{l}}\right)^{\frac{1}{q} - \frac{1}{p}} \frac{\|\tilde{\boldsymbol{\omega}}_{k,E_{l}}^{l,t}\|_p}{\|\tilde{\boldsymbol{\omega}}_{k,E_{l}}^{l,t}\|_q}$
        \STATE Compute the lower boundary required model parameters to keep $r_{t}^{l} = d_{t}^{l}(1 + \eta_c)^{-\frac{q}{q-p}}\left[1 - I(\tilde{\boldsymbol{\omega}}_{k,E_{l}}^{l,t})\right]^{\frac{p}{q-p}}$
        \STATE Compute the number of model parameters to prune
        \STATE $c_{t}^{l} = \left\lfloor d_{t}^{l} \cdot \min\left(\gamma\left(1 - \frac{r_{t}^{l}}{d_{t}^{l}}\right), \beta\right) \right\rfloor$
        \STATE Prune $c_{t}^{l}$ model parameters with the smallest magnitude based on $\tilde{\boldsymbol{\omega}}_{k,E_{l}}^{l,t}$ and $\mathbf{m}_{k}^{l,t}$
        \STATE Find new layer mask $\mathbf{m}_{k}^{l,t\prime}$ and pruned model $\tilde{\boldsymbol{\omega}}_{k,E_{l}}^{l, t\prime}$ at layer $l$
    \ENDFOR
    \STATE Obtain $\tilde{\boldsymbol{\omega}}_{k,E_{l}}^{t\prime}$ and corresponding mask $\mathbf{m}_{k}^{t\prime}$
\ENDFOR
\end{algorithmic}
\end{algorithm}

\subsection{RigL Algorithm}
\label{sec:rigl}
We follow the RigL algorithm to generate the new mask each communication round, which is shown in Algorithm \ref{algorithm:rigl}.

\begin{algorithm}
\caption{RigL mask generation}
\label{algorithm:rigl}
\begin{algorithmic}[1]
\STATE \textbf{Input:} $\tilde{\boldsymbol{\omega}}_{k,E_{l}}^{t\prime}$, corresponding mask $\mathbf{m}_{k}^{t\prime}$, global rounds T, initial annealing ratio $\alpha_{0}$
\STATE \textbf{Output:} New mask $\mathbf{m}_{k}^{t+1}$
\STATE Compute prune ratio $\alpha_t= \frac{\alpha}{2} \left(1 + \cos\left(\frac{t\pi}{T}\right)\right)$ 
\STATE Sample one batch of local training data to calculate dense gradient $\mathbf{g}(\tilde{\boldsymbol{\omega}}_{k,E_{l}}^{t\prime})$
\FOR{each layer $l \in |L|$}
    \STATE Update mask $\mathbf{m}_{k}^{l,t+\frac{1}{2}\prime}$ by pruning $\alpha_t$ percentage of weights based on weight magnitude.
    \STATE Update mask $\mathbf{m}_{k}^{l,t+1}$ via regrowing weights with gradient information $\mathbf{g}(\tilde{\boldsymbol{\omega}}_{k,E_{l}}^{t\prime})$.
\ENDFOR
\STATE Find new mask $\mathbf{m}_{k}^{t+1}$.
\end{algorithmic}
\end{algorithm}

\label{appendix:baselines}
\section*{Convergence Analysis}\label{appendice_proof}
In a time-varying connected topology, both $\mathcal{G}_{k}^{t}$ and $\mathcal{N}_{k}^{t}$ are randomly generated. We consider $\mathcal{N}_{k}^{t}=\mathcal{G}_{k}^{t}$ in theoretical analysis since our scheduling policy is regarded as one type of client selection policy.

\subsection{Client Selection Analysis}\label{appendix:client_hypergeometric}

Given a system with \( K \) clients with indices sorted from $1$ to $K$, and considering a particular client with index \( k \) then: 
(1) there are \( K-1 \) potential clients to select from; (2) among these \( K-1\)  clients, \( k-1 \) clients have an index less than \( k \); (3) we wish to select \( M \) total clients in each sample as client's $k$ neighbors. Hence, $|\mathcal{N}_{(a)k}|$ is a hypergeometric random variable and the probability \( \mathbb{P}(m, k) = \mathbb{P}(|\mathcal{N}_{(a)k}| = m)\) that exactly \( m \) of the selected clients have an index less than \( k \) is
\begin{equation}
\mathbb{P}(m, k) = \frac{\binom{k-1}{m} \binom{K - k}{M - m}}{\binom{K-1}{M}},
\end{equation}
where $\sum_{0\leq m \leq \min_{(M,k-1)}} \mathbb{P}(m,k)=1$, which essentially follows from Vandermonde's identity.
\subsection{Auxiliary Lemmas and Proofs}
DA-DPFL's local update follows:
\begin{equation}\label{eq:local_update}
    \Tilde{\boldsymbol{\omega}}_{k,\tau+1}^{t} = \Tilde{\boldsymbol{\omega}}_{k,\tau}^{t} - \eta \mathbf{g}_{k,\tau}^{t}\odot \mathbf{m}_{k}^{t},
\end{equation}
where $\mathbf{g}_{k,\tau}^{t}=\nabla F_{k}(\Tilde{\boldsymbol{\omega}}_{k,\tau}^{t})$. This implies that
\begin{equation}\label{eq:local_train_mask}
    \eta \sum_{\tau=0}^{E_{l}-1}\mathbf{g}_{k,\tau}^{t}\odot \mathbf{m}_{k}^{t} = (\Tilde{\boldsymbol{\omega}}_{k,0}^{t}-\Tilde{\boldsymbol{\omega}}_{k,E_{l}}^{t})\odot \mathbf{m}_{k}^{t}.
\end{equation}
Note that $\boldsymbol{\omega}_{k}^{t+1}=\Tilde{\boldsymbol{\omega}}^{t}_{k,E_{l}}$ and $\Tilde{\boldsymbol{\omega}}_{k}^{t}=\Tilde{\boldsymbol{\omega}}_{k,0}^{t}$. Considering traditional aggregation, like in FedAvg, \textbf{without} scheduling first, we then have Lemma \ref{lemma2:boundary_local_updates}.

\begin{lemma}\label{lemma2:boundary_local_updates}
     Under Assumptions~\ref{assumption:mu_lips}~to~\ref{assumption:bounded_var}, for some $M>1$ and $\eta$ such that $\eta^{2} \leq \frac{1}{12M\mu^{2}(M-1)(2M-1)}$,
\begin{align}\label{eq:lemma2_thm1_0}
&\frac{1}{K}\sum_{k=1}^{K}\mathbb{E}\|\boldsymbol{\omega}_{k}^{t+1} - \Tilde{\boldsymbol{\omega}}_{k}^{t}\|^{2} 
\leq \left( e^{\frac{E_{l}}{2M-2}} - 1 \right) (2M-2) \nonumber \\
&\quad \times \left( \frac{2M}{2M-1}\eta^{2}\sigma^{2}_{l} + 6M\eta^{2}\sigma^{2}_{g} \right. \nonumber \\
&\quad \left. + 6M\eta^{2}\frac{\sum_{k=1}^{K}\mathbb{E}\|\nabla f(\Tilde{\boldsymbol{\omega}}_{k}^{t})\|^{2}}{K} \right).
\end{align}

 \end{lemma}
\begin{proof}
 Because the mask $\mathbf{m}_{k}^{t}$ is consistent during training, we omit the expression with the corresponding model for brevity.   
 We firstly consider the traditional weighted average aggregation where

\begin{align}\label{eq:lemma2_thm1_1}
&\mathbb{E}\|\Tilde{\boldsymbol{\omega}}_{k,\tau+1}^{t} - \Tilde{\boldsymbol{\omega}}_{k}^{t}\|^{2} 
= \mathbb{E}\Big\| \tilde{\boldsymbol{\omega}}_{k,\tau}^{t} - \Tilde{\boldsymbol{\omega}}_{k}^{t} 
-\eta \big( g_{k,\tau}^{t}\odot \mathbf{m}_{k}^{t}  \nonumber \\
&\quad - \nabla f_{k}(\tilde{\boldsymbol{\omega}}_{k,\tau}^{t})+ \nabla f_{k}(\tilde{\boldsymbol{\omega}}_{k,\tau}^{t}) - \nabla f(\tilde{\boldsymbol{\omega}}_{k}^{t}) 
+ \nabla f(\tilde{\boldsymbol{\omega}}_{k}^{t}) \nonumber \\
&\quad - \nabla f_{k}(\tilde{\boldsymbol{\omega}}_{k}^{t})+ \nabla f_{k}(\tilde{\boldsymbol{\omega}}_{k}^{t})\big) \Big\|^{2}.
\end{align}

Write $a:=\mathbb{E}\big\|\tilde{\boldsymbol{\omega}}_{k,\tau}^{t} - \eta \big(\tilde{g}_{k,\tau}^{t}\odot \mathbf{m}_{k}^{t} - \nabla f_{k}(\tilde{\boldsymbol{\omega}}_{k,\tau}^{t})\big)- \Tilde{\boldsymbol{\boldsymbol{\omega}}}_{k}^{t} \big\|^{2}$ and $b= \eta^{2}\mathbb{E}\|\nabla f_{k}(\tilde{\boldsymbol{\omega}}_{k,\tau}^{t}) - \nabla f(\tilde{\boldsymbol{\omega}}_{k}^{t})+ \nabla f(\tilde{\boldsymbol{\omega}}_{k}^{t})- \nabla f_{k}(\tilde{\boldsymbol{\omega}}_{k}^{t})+ \nabla f_{k}(\tilde{\boldsymbol{\omega}}_{k}^{t}) \|^{2}$, using the Cauchy’s inequality with a elastic variable $2M=2M>1$, we have 
\begin{equation}\label{eq:lemma2_thm1_2}
    \mathbb{E}\|\boldsymbol{\omega}_{k}^{t+1} - \Tilde{\boldsymbol{\omega}}_{k}^{t}\|^{2} \leq (1+\frac{1}{2M-1})a+2Mb.
\end{equation}
Then,  by Assumptions \ref{assumption:mu_lips} to \ref{assumption:bounded_var} and the triangle inequality, 
\begin{align}\label{eq:lemma2_thm1_3}
    a&\leq \mathbb{E}\|\tilde{\boldsymbol{\omega}}_{k,\tau}^{t}-\Tilde{\boldsymbol{\omega}}_{k}^{t}\|^{2}+\eta^{2}\mathbb{E}\|\tilde{g}_{k,\tau}^{t}\odot \mathbf{m}_{k}^{t} - \nabla f_{k}(\tilde{\boldsymbol{\omega}}_{k,\tau}^{t})\|^{2}\notag\\
    &= \mathbb{E}\|\tilde{\boldsymbol{\omega}}_{k,\tau}^{t}-\Tilde{\boldsymbol{\omega}}_{k}^{t}\|^{2} + \eta^{2}\sigma^{2}_{l},
\end{align}
and 
\begin{align}\label{eq:lemma2_thm1_4}
    b &\leq 3\eta^{2}\big[\mathbb{E}\|\nabla f_{k}(\tilde{\boldsymbol{\omega}}_{k,\tau}^{t}) - \nabla f_{k}(\tilde{\boldsymbol{\omega}}_{k}^{t})\|^{2} \nonumber \\
    &\quad +\mathbb{E}\|\nabla f(\tilde{\boldsymbol{\omega}}_{k}^{t})\|^{2} + \mathbb{E}\|\nabla f_{k}(\tilde{\boldsymbol{\omega}}_{k}^{t}) - \nabla f(\tilde{\boldsymbol{\omega}}_{k}^{t})\|^{2}\big] \notag\\
    &\leq 3\eta^{2}\big[\mathbb{E}\|\nabla f_{k}(\tilde{\boldsymbol{\omega}}_{k,\tau}^{t}) - \nabla f_{k}(\tilde{\boldsymbol{\omega}}_{k}^{t})\|^{2} \nonumber \\
    &\quad +\mathbb{E}\|\nabla f(\tilde{\boldsymbol{\omega}}_{k}^{t})\|^{2} + \mathbb{E}\|\nabla f_{k}(\tilde{\boldsymbol{\omega}}_{k}^{t}) - \nabla f(\tilde{\boldsymbol{\omega}}^{t})\|^{2} \nonumber \\
    &\quad +\mathbb{E}\|\nabla f(\tilde{\boldsymbol{\omega}}_{k}^{t}) - \nabla f(\tilde{\boldsymbol{\omega}}^{t})\|^{2}\big] \notag\\
    &\leq 3\eta^{2}\big[\mu^{2} \mathbb{E}\|\tilde{\boldsymbol{\omega}}_{k,\tau}^{t}-\Tilde{\boldsymbol{\omega}}_{k}^{t}\|^{2} + \mathbb{E}\|\nabla f(\tilde{\boldsymbol{\omega}}_{k}^{t})\|^{2} \nonumber \\
    &\quad +\sigma^{2}_{g}\big].
\end{align}

Substitute Eq.\eqref{eq:lemma2_thm1_3} \& \eqref{eq:lemma2_thm1_4} into Eq. \eqref{eq:lemma2_thm1_2} with some $\eta$ such that $\eta^{2} \leq \frac{1}{12M\mu^{2}(M-1)(2M-1)}$, we have
\begin{align}\label{eq:lemma2_thm1_5}
&\mathbb{E}\|\Tilde{\boldsymbol{\omega}}_{k,\tau+1}^{t} - \Tilde{\boldsymbol{\omega}}_{k}^{t}\|^{2} 
\leq (1+\frac{1}{2M-1} + 6M\eta^{2}\mu^{2})\mathbb{E}\|\tilde{\boldsymbol{\omega}}_{k,\tau}^{t} - \Tilde{\boldsymbol{\omega}}_{k}^{t}\|^{2} \notag \\
&\quad + (1+\frac{1}{2M-1})\eta^{2}\sigma^{2}_{l} + 6M\eta^{2}(\sigma^{2}_{g}+\mathbb{E}\|\nabla f(\tilde{\boldsymbol{\omega}}_{k}^{t})\|^{2}) \notag \\
&\leq (1+\frac{1}{2M-2})\mathbb{E}\|\tilde{\boldsymbol{\omega}}_{k,\tau}^{t} - \Tilde{\boldsymbol{\omega}}_{k}^{t}\|^{2} \notag \\
&\quad + (1+\frac{1}{2M-1})\eta^{2}\sigma^{2}_{l} + 6M\eta^{2}(\sigma^{2}_{g}+\mathbb{E}\|\nabla f(\tilde{\boldsymbol{\omega}}_{k}^{t})\|^{2}) 
\end{align}

Let \( A = 1 + \frac{1}{2(M-1)} \), \( B = \left(1+\frac{1}{2M-1}\right)\eta^{2}\sigma^{2}_{l} + 6M\eta^{2}\sigma^{2}_{g} \), and \( C = 6M\eta^{2}\mathbb{E}\|\nabla f(\tilde{\boldsymbol{\omega}}_{k}^{t})\|^{2} \), then the recursive inequality Eq.\eqref{eq:lemma2_thm1_5} becomes
\begin{align}\label{eq:lemma2_thm1_6}
\mathbb{E}\|\Tilde{\boldsymbol{\omega}}_{k,\tau+1}^{t} - \Tilde{\boldsymbol{\omega}}_{k}^{t}\|^{2} &\leq A\mathbb{E}\|\tilde{\boldsymbol{\omega}}_{k,\tau}^{t} - \Tilde{\boldsymbol{\omega}}_{k}^{t}\|^{2} + B + C.
\end{align}

When $\tau = 0$, the initial condition is $\mathbb{E}\|\tilde{\boldsymbol{\omega}}_{k,0}^{t} - \Tilde{\boldsymbol{\omega}}_{k}^{t}\|^{2} = 0$. For $\tau = 1$ to $E_l$, we apply the inequality Eq.\eqref{eq:lemma2_thm1_6} $E_l$ times, summing up the constants multiplied by their respective powers of $A$ gives
\begin{align}
\mathbb{E}\|\Tilde{\boldsymbol{\omega}}_{k,E_l}^{t} - \Tilde{\boldsymbol{\omega}}_{k}^{t}\|^{2} &\leq A^{E_l}\mathbb{E}\|\tilde{\boldsymbol{\omega}}_{k,0}^{t} - \Tilde{\boldsymbol{\omega}}_{k}^{t}\|^{2} + B\sum_{j=0}^{E_l-1} A^j \notag\\
&+ C\sum_{j=0}^{E_l-1} A^j.
\end{align}

The sums of the series can be simplified by the sum of a geometric series as follows
\begin{align}
\sum_{j=0}^{E_l-1} A^j = \frac{1 - A^{E_l}}{1 - A}.
\end{align}

Hence the inequality can be further simplified as 
\begin{align}
\mathbb{E}\|\Tilde{\boldsymbol{\omega}}_{k,E_l}^{t} - \Tilde{\boldsymbol{\omega}}_{k}^{t}\|^{2} &\leq 0 + (B+C)\frac{A^{E_l}-1}{A-1}.
\end{align}
When $M>1$, $ A = 1 + \frac{1}{2M-2} < e^{\frac{1}{2M-2}}$ hence $A^{E_{l}}<e{^\frac{E_{l}}{2M-2}}$, which gives the final bound for \( \mathbb{E}\|\boldsymbol{\omega}_{k}^{t+1} - \Tilde{\boldsymbol{\omega}}_{k}^{t}\|^{2} \) as in Eq.\eqref{eq:lemma2_thm1_0}. 
\end{proof}

\begin{lemma}\label{lemma5}
Consider the proposed scheduling strategy. Let $\Tilde{\boldsymbol{\omega}}_{k}^{t(\dag)}$ denote the local personalized aggregated model for the $k$-th client at time $t$. The global aggregated model at time $t$, $\Tilde{\boldsymbol{\omega}}^{t(\dag)}$, is defined as the average of the local models, i.e., $\Tilde{\boldsymbol{\omega}}^{t(\dag)} = \frac{1}{K}\sum_{k=1}^{K}\Tilde{\boldsymbol{\omega}}_{k}^{t(\dag)}$, where $K$ is the total number of clients. Let $M$ represent the number of clients in the neighborhood. With the support of Lemma $1$ in \cite{wan2021convergence}, the expected value of the global model at time $t+1$, denoted as $\mathbb{E}(\Tilde{\boldsymbol{\omega}}^{t+1(\dag)})$, is given by

\begin{align}
    \mathbb{E}(\Tilde{\boldsymbol{\omega}}^{t+1(\dag)}) = \mathbb{E}(\Tilde{\boldsymbol{\omega}}^{t(\dag)}) - \eta\frac{\mathbb{E}\left( \sum_{k=1}^{K}\sum_{\tau=0}^{E_{l}^{*}-1}\mathbf{g}_{\tau,k}(\Tilde{\boldsymbol{\omega}}^{t(\dag)}) \right)}{K},
\end{align}
where $E_{l}^{*} = \frac{3M+2}{2(M+1)}E_{l}$, and $E_{l}$ is the number of steps of local updates. Here, $\mathbf{g}_{\tau,k}$ represents the gradient computation for the $k$-th client at local update step $\tau$. %This expression encapsulates the dynamic nature of the global model's evolution over time, factoring in the contributions and updates from each participating client.
\end{lemma}

\begin{proof}

Now we consider the effect of DA-DPFL's scheduling. Rewrite the mask element aggregation with the sequential appointment as 
\begin{align}\label{eq:theorem1_1}
     \Tilde{\boldsymbol{\omega}}_{k}^{t(\dag)} &= \left( \frac{\sum_{j\in \mathcal{N}_{(a)k}^{t}}\boldsymbol{\omega}_{j}^{t}+\sum_{j\in \mathcal{N}_{(b)k}^{t}}\boldsymbol{\omega}_{j}^{t}+\boldsymbol{\omega}_{k}^{t}}{\sum_{j\in \mathcal{N}_{(a)k}^{t}}\mathbf{m}_{j}^{t}+\sum_{j\in \mathcal{N}_{(b)k}^{t}}\mathbf{m}_{j}^{t}+\mathbf{m}_{k}^{t}}\right)\odot \mathbf{m}_{k}^{t}\\
     &= \left( \frac{\sum_{j\in \mathcal{N}_{(a)k}^{t}}\boldsymbol{\omega}_{j}^{t}+\sum_{j\in \mathcal{N}_{(b)k+}^{t}}\boldsymbol{\omega}_{j}^{t}}{M+1}\right)\odot \mathbf{m}_{k}^{t},
     %= \left(\frac{\boldsymbol{\omega}_{k}^{t}+\sum_{j\in \mathcal{N}_{k}^{t}} \sum_{m\in \lvert M \rvert}P(m,j) \boldsymbol{\omega}_{j}^{t}}{\mathbf{m}_{k}^{t}+\sum_{j\in \mathcal{N}_{k}^{t}}\sum_{m\in \lvert M \rvert}P(m,j) \mathbf{m}_{j}^{t}}\right)\odot \mathbf{m}_{k}^{t}
\end{align}
where $\mathcal{N}_{(b)k}^{t}:=\mathcal{N}_{k}^{t}\setminus\mathcal{N}_{(a)k}^{t}$, $\mathcal{N}_{(b)k+}^{t} \coloneqq \mathcal{N}_{(b)k}^{t}\cup \{k\}$, and %$k \in \mathcal{N}_{(b)k+}^{t}$ is according to the definition and 
the last equation holds under Assumption \ref{assumption:mask_influence}.

Similar to Eq.\eqref{eq:local_train_mask}, we omit $\mathbf{m}_{k}^{t}$ for convenience on notations since the mask is consistent during local training. Then for the $j$-th client at time $t$, the local personalized aggregated model is
\begin{equation}\label{eq:theorem1_3}
\boldsymbol{\omega}_{j}^{t(\dag)} = 
\begin{cases}
\Tilde{\boldsymbol{\omega}}_{j}^{t}-\eta \sum_{\tau=0}^{E_{l}-1}\mathbf{g}_{j,\tau}^{t}\odot \mathbf{m}_{j}^{t}, & \text{if } j \in\mathcal{N}_{(a)k}^{t};\\
\boldsymbol{\omega}_{j}^{t}, & \text{otherwise}.
\end{cases}
\end{equation}
 When $j \in\mathcal{N}_{(a)k}^{t}$, one can see that $\boldsymbol{\omega}_{j}^{t(\dag)}$ is equivalent to $\boldsymbol{\omega}_{j}^{t+1}$, reflecting the scenario where, within a single communication round, all participating clients perform an equal number of local training iterations, \textbf{analogous to traditional FL}. The superscript $(\dag)$ is introduced for an explicit differentiation, signifying that although the local gradients $\mathbf{g}_{j,\tau}^{t}$ are computed under varying aggregation models, they are \textbf{distinct} from those in a synchronous FL framework.
Denoted by $I_{\{j \in \mathcal{N}^t_{(a)k}\}}$ an indicator function for the event that the $j$-th client is selected in the \textit{delayed} neighborhoods of client $k$. %\textcolor{red}{jump to Eq.\ref{eq:theorem1_5}}

Using the results of Proposition \ref{proposition1}, it can be shown that
\begin{align}\label{eq:theorem1_5}
    \mathbb{E} ( \sum_{j\in \mathcal{N}_{(a)k}^{t}} \boldsymbol{\omega}_{j}^{t+1} )
    &= \mathbb{E} ( \mathbb{E} ( \sum_{j\in \mathcal{N}_{(a)k}^{t}} \boldsymbol{\omega}_{j}^{t+1} | |\mathcal{N}^t_{(a)k}| ) )\notag\\
    &= \mathbb{E} ( \mathbb{E} ( \sum_{j\in \mathcal{N}_{k}^{t}} \boldsymbol{\omega}_{j}^{t+1} I_{\{j \in \mathcal{N}^t_{(a)k}\}} | |\mathcal{N}^t_{(a)k}| ) )\notag\\
    &= \mathbb{E} ( \mathbb{E} ( \sum_{j\in \mathcal{N}_{k}^{t}} \boldsymbol{\omega}_{j}^{t+1} I_{\{j \in \mathcal{N}^t_{(a)k} | |\mathcal{N}^t_{(a)k}|\}}  ) )\notag\\
    &= \mathbb{E} (   \sum_{j\in \mathcal{N}_{k}^{t}} \boldsymbol{\omega}_{j}^{t+1} \mathbb{E}  (I_{\{j \in \mathcal{N}^t_{(a)k} | |\mathcal{N}^t_{(a)k}|\}}  ) )\notag\\
    &= \mathbb{E} (   \sum_{j\in \mathcal{N}_{k}^{t}} \boldsymbol{\omega}_{j}^{t+1} \mathbb{P}  (j \in \mathcal{N}^t_{(a)k} | |\mathcal{N}^t_{(a)k}|  ) )\notag\\
    &= \mathbb{E} (   \sum_{j\in \mathcal{N}_{k}^{t}} \boldsymbol{\omega}_{j}^{t+1} \frac{|\mathcal{N}^t_{(a)k}|}{M}  )\notag\\
    &= \mathbb{E} ( |\mathcal{N}^t_{(a)k}|  )  \frac{\mathbb{E} (\sum_{j\in \mathcal{N}_{k}^{t}} \boldsymbol{\omega}_{j}^{t+1})}{M} \notag\\
    &= \frac{(k-1)M}{K-1}\mathbb{E} (\Bar{\boldsymbol{\omega}}^{t+1}),
\end{align}
where $|\mathcal{N}^t_{(a)k}|$ follows hypergeometric distribution. Similarly,
\begin{align}\label{eq:theorem1_6}
    \mathbb{E} ( \sum_{j\in \mathcal{N}_{(b)k+}^{t}} \boldsymbol{\omega}_{j}^{t} )
    &= \mathbb{E} ( \mathbb{E} ( \sum_{j\in \mathcal{N}_{(b)k+}^{t}} \boldsymbol{\omega}_{j}^{t} | |\mathcal{N}^t_{(b)k+}| ) )\notag\\
    &= \mathbb{E} ( \mathbb{E} ( \sum_{j\in \mathcal{N}_{k+}^{t}} \boldsymbol{\omega}_{j}^{t} I_{\{j \in \mathcal{N}^t_{(b)k+}\}} | |\mathcal{N}^t_{(b)k+}| ) )\notag\\
    &= \mathbb{E} ( \mathbb{E} ( \sum_{j\in \mathcal{N}_{k+}^{t}} \boldsymbol{\omega}_{j}^{t} I_{\{j \in \mathcal{N}^t_{(b)k+} | |\mathcal{N}^t_{(b)k+}|\}}  ) )\notag\\
    &= \mathbb{E} (   \sum_{j\in \mathcal{N}_{k+}^{t}} \boldsymbol{\omega}_{j}^{t} \mathbb{E}  (I_{\{j \in \mathcal{N}^t_{(b)k+} | |\mathcal{N}^t_{(b)k+}|\}}  ) )\notag\\
    &= \mathbb{E} (   \sum_{j\in \mathcal{N}_{k+}^{t}} \boldsymbol{\omega}_{j}^{t} \mathbb{P}  (j \in \mathcal{N}^t_{(b)k+} | |\mathcal{N}^t_{(b)k+}|  ) )\notag\\
    &= \mathbb{E} (   \sum_{j\in \mathcal{N}_{k+}^{t}} \boldsymbol{\omega}_{j}^{t} \frac{|\mathcal{N}^t_{(b)k+}|-1}{M}  )\notag\\
    &= \mathbb{E} ( (|\mathcal{N}^t_{(b)k+}|-1)  )  \frac{ \mathbb{E} (\sum_{j\in \mathcal{N}_{k+}^{t}}   \boldsymbol{\omega}_{j}^{t})}{M} \notag\\
    &= \frac{(K-k)(M+1)}{(K-1)}\mathbb{E} (\Bar{\boldsymbol{\omega}}^{t}) .
\end{align}
Therefore, $\mathbb{E}(\Tilde{\boldsymbol{\omega}}_{k}^{t(\dag)})$ can be written as 
%subsequently, \textcolor{red}{(need to double check the following results using the mean of hypergeometrical distribution)}
\begin{align}\label{eq:theorem1_7}
    &[\frac{(k-1)M}{(K-1)(M+1)}\mathbb{E}(\Bar{\boldsymbol{\omega}}^{t+1})   +  \frac{(K-k)}{(K-1)}\mathbb{E}(\Bar{\boldsymbol{\omega}}^{t}) ] \odot \mathbf{m}_{k}^{t} \notag \\
    &= \mathbb{E}(\Bar{\boldsymbol{\omega}}^{t}) \odot \mathbf{m}_{k}^{t}  - [\frac{(k-1)M}{(K-1)(M+1)} \times . \notag \\
    &\quad \mathbb{E}(\frac{\eta\sum_{j\in \mathcal{N}_{k+}^{t}}\sum_{\tau=0}^{E_{l}-1}g_{j,\tau}^{t}}{M+1})] \odot \mathbf{m}_{k}^{t},
\end{align}

where one can verify that when $k=1$, $\mathbb{E} ( \Tilde{\boldsymbol{\omega}}_{k}^{t(\dag)} )$ reduces to $\mathbb{E} ( \Tilde{\boldsymbol{\omega}}_{k}^{t} )$. 
Recall that $\Tilde{\boldsymbol{\omega}}^{t(\dag)}=\frac{1}{K}\sum_{k=1}^{K}\Tilde{\boldsymbol{\omega}}_{k}^{t(\dag)}$, then
\begin{align}\label{eq:omega_final_expression}
    &\mathbb{E}(\Tilde{\boldsymbol{\omega}}^{t(\dag)}) 
    = \frac{1}{K}\sum_{k=1}^{K}(\mathbb{E}(\Bar{\boldsymbol{\omega}}^{t} \notag \\
    &\quad -\frac{(k-1)M}{(K-1)(M+1)}\eta\sum_{j\in \mathcal{N}_{k+}^{t}}\sum_{\tau=0}^{E_{l}-1}g_{j,\tau}^{t}(\Bar{\boldsymbol{\omega}}^{t})) \notag \odot \mathbf{m}_{k}^{t}) \\
    &= \frac{1}{K}\sum_{k=1}^{K}(\mathbb{E}(\Bar{\boldsymbol{\omega}}^{t})  -\frac{(k-1)M}{(K-1)(M+1)}\eta \mathbb{E}(\Tilde{g}_{k}^{t})\odot \mathbf{m}_{k}^{t}) \\
    &= \mathbb{E}(\Bar{\boldsymbol{\omega}}^{t}) -\frac{1}{K}\sum_{k=1}^{K}(\frac{(k-1)M}{(K-1)(M+1)}\eta \mathbb{E}(\Tilde{g}_{k}^{t}))\odot \mathbf{m}_{k}^{t},
\end{align}

where $\Tilde{g}_{k}^{t}=\frac{\sum_{j\in \mathcal{N}_{k+}^{t}}\sum_{\tau=0}^{E_{l}-1}g_{j,\tau}^{t}}{M+1}$ and the last equality holds according to the definition of $\Bar{\boldsymbol{\omega}}^{t}$. $g_{j,\tau}^{t}(\Bar{\boldsymbol{\omega}})$ means the gradient at local epoch $\tau=0$ is with respect to $\Bar{\boldsymbol{\omega}}$. 
Let $\Tilde{g}^{t(\dag)}=\frac{1}{K}\sum_{k=1}^{K}[\frac{(k-1)M}{(K-1)(M+1)}\Tilde{g}_{k}^{t}\odot \mathbf{m}_{k}^{t}]$, then
 \begin{align}\label{eq:new_agg_expectation}
     \mathbb{E}(\Tilde{\boldsymbol{\omega}}^{t(\dag)}) &= \mathbb{E}(\Bar{\boldsymbol{\omega}}^{t})-\eta \mathbb{E}(\Tilde{g}^{t(\dag)}).
 \end{align}
 %since the $\frac{1}{K}\sum_{k=1}^{K}\frac{k-1}{K}=\frac{1}{2}$. This implies our new scheduling strategy for the global model $\Tilde{\boldsymbol{\omega}}^{t(\dag)}$ sits in the interval between $\Bar{\boldsymbol{\omega}}^{t}$ and $\Bar{\boldsymbol{\omega}}^{t+1}$.
 
 To find the boundary for the difference between the global model at time $t$ and $t+1$ and according to Assumption \ref{assumption:mu_lips}, we have 
\begin{align}\label{eq:difference_global_model}
    &\mathbb{E}[f(\Tilde{\boldsymbol{\omega}}^{t+1(\dag)})] -\mathbb{E}[f(\Tilde{\boldsymbol{\omega}}^{t(\dag)})] \notag \\
    &\leq \mathbb{E}\left[\left<f(\nabla\Tilde{\boldsymbol{\omega}}^{t(\dag)}), \Tilde{\boldsymbol{\omega}}^{t+1(\dag)}-\Tilde{\boldsymbol{\omega}}^{t(\dag)}\right>\right] \notag \\
    &\quad +\frac{\mu}{2}\|\Tilde{\boldsymbol{\omega}}^{t+1(\dag)}-\Tilde{\boldsymbol{\omega}}^{t(\dag)}\|^{2}.
\end{align}

 %because (\ref{eq:local_train_mask}). \textcolor{red}{Let's try to prove this first as a basic solution to show how scheduling affects convergence}

Lemma $1$ in \cite{wan2021convergence} ensures that $\forall k \in \lvert K \rvert,\ \mathbb{E}\left(\frac{ \sum_{k=1}^{K}\sum_{\tau=0}^{E_{l}-1}\mathbf{g}_{\tau,k}(\Tilde{\boldsymbol{\omega}}^{t(\dag)}) }{K}\right)=\mathbb{E}[\mathbf{\Tilde{g}}_{k}^{t}]$ since $\mathcal{N}_{k}^{t}$ is selected randomly.
According to the definition and Eq.~\eqref{eq:new_agg_expectation},
\begin{align}\label{eq:t_tplus1_agg_relation}
    \mathbb{E}(\Tilde{\boldsymbol{\omega}}^{t+1(\dag)}) 
    &= \mathbb{E}(\Bar{\boldsymbol{\omega}}^{t+1}) - \eta \mathbb{E}(\Tilde{g}^{t+1(\dag)}) \notag \\
    &= \mathbb{E}(\Tilde{\boldsymbol{\omega}}^{t(\dag)}) - \eta\frac{\mathbb{E}\left( \sum_{k=1}^{K}\sum_{\tau=0}^{E_{l}-1}\mathbf{g}_{\tau,k}(\Tilde{\boldsymbol{\omega}}^{t(\dag)}) \right)}{K} \notag \\
    &\quad -\eta \mathbb{E}(\Tilde{g}^{t+1(\dag)}) \\
    &= \mathbb{E}(\Tilde{\boldsymbol{\omega}}^{t(\dag)}) - \eta\frac{\mathbb{E}\left( \sum_{k=1}^{K}\sum_{\tau=0}^{E_{l}^{*}-1}\mathbf{g}_{\tau,k}(\Tilde{\boldsymbol{\omega}}^{t(\dag)}) \right)}{K},
\end{align}

where $E_{l}^{*}=E_{l}+\frac{M}{2(M+1)}E_{l}$.
Actually, here the starting weights of $\Tilde{g}^{t+1(\dag)}$ is with respect to $\mathbb{E} (\Bar{\boldsymbol{\omega}}^{t+1})=\ \mathbb{E}(\Tilde{\boldsymbol{\omega}}^{t(\dag)}) -\eta \frac{\sum_{k=1}^{K}\mathbb{E}\sum_{\tau=0}^{E_{l}-1}\left(\mathbf{g}_{\tau,k}(\Tilde{\boldsymbol{\omega}}^{t(\dag)}) \right)}{K}$. One can think of just continuing $\frac{M}{2(M+1)}$ more steps of local training, where $\frac{1}{K}\sum_{k=1}^{K}[\frac{(k-1)M}{(K-1)(M+1)}]=\frac{M}{2(M+1)}$. $E_{l}^{*}$ might not be an integer, we just use this to conclude the effect of sequential aligning for convergence analysis.
\end{proof}
%Note that the target is to prove the boundary for $\Tilde{\boldsymbol{\omega}}_{k}^{t+1}-\Tilde{\boldsymbol{\omega}}_{k}^{t}$. With the aid of Eq. \eqref{eq:theorem1_4},
%\begin{align}
%    &\Tilde{\boldsymbol{\omega}}_{k}^{t+1} - \Tilde{\boldsymbol{\omega}}_{k}^{t} 
%    = \bar{\boldsymbol{\omega}}_{k}^{t+1} - \bar{\boldsymbol{\omega}}_{k}^{t} \\
%    &\quad + \frac{\mathbb{E}_{P(j)}\left[\sum_{j \in \mathcal{N}_{(a)k}^{t+1}} 
%    (\Tilde{\boldsymbol{\omega}}_{j}^{t+1} - \boldsymbol{\omega}_{j}^{t+1})-\sum_{j \in \mathcal{N}_{(a)k}^{t}}(\Tilde{\boldsymbol{\omega}}_{j}^{t} - \boldsymbol{\omega}_{j}^{t})\right]
%    }{M+1} \\
%    &\quad - \frac{\mathbb{E}_{P(m,j)}\left[ \sum_{m \in \lvert M \rvert} 
%    \sum_{j \in \mathcal{N}_{(a)k}^{t+1}} \eta \sum_{\tau = 0}^{E_{l}-1} 
%    P(m,j)\mathbf{g}_{j,\tau}^{t+1}\odot \mathbf{m}_{j}^{t+1}\right] }{M+1} \\
%    &\quad + \frac{\mathbb{E}_{P(m,j)}\left[ \sum_{m \in \lvert M \rvert} 
%    \sum_{j \in \mathcal{N}_{(a)k}^{t}} \eta \sum_{\tau = 0}^{E_{l}-1} 
%    P(m,j)\mathbf{g}_{j,\tau}^{t} \odot \mathbf{m}_{j}^{t}\right]}{M+1}
%\end{align}
%where $\bar{\boldsymbol{\omega}}_{k}^{t}=\mathbb{E}_{P(j)}[\frac{1}{M+1}\sum_{j\in \mathcal{N}_{k+}^{t}}\boldsymbol{\omega}_{j}^{t}]=\frac{\sum_{j\in \mathcal{N}_{k+}^{t}}\boldsymbol{\omega}_{j}^{t}}{M+1}$. Then,
%\begin{align}
%    \bar{\boldsymbol{\omega}}_{k}^{t+1} - \bar{\boldsymbol{\omega}}_{k}^{t}&=
%    -\frac{\eta \sum_{j\in \mathcal{N}_{k+}^{t}}\sum_{\tau = 0}^{E_{l}-1} 
%    \mathbf{g}_{j,\tau}^{t}\odot \mathbf{m}_{j}^{t}}{M+1}\\
%\end{align}
Now, the target is to prove the boundary of $\mathbb{E}<f(\nabla\Tilde{\boldsymbol{\omega}}^{t(\dag)}),\Tilde{\boldsymbol{\omega}}^{t+1(\dag)}-\Tilde{\boldsymbol{\omega}}^{t(\dag)}>$ and $\frac{\mu}{2}\|\Tilde{\boldsymbol{\omega}}^{t+1(\dag)}-\Tilde{\boldsymbol{\omega}}^{t(\dag)}\|^{2}$.
\begin{lemma}\label{lemma4}
     Under Assumptions~\ref{assumption:mu_lips}~to~\ref{assumption:bounded_var} and Lemma \ref{lemma5}, suppose $\Tilde{\boldsymbol{\omega}}^{t+1(\dag)}$ and $\Tilde{\boldsymbol{\omega}}^{t(\dag)}$ are global model learned by the proposed strategy, for some $M>1$ with $\eta \leq \sqrt{\frac{1}{12M\mu^{2}(M-1)(2M-1)}}$, $\mathbb{E}\left [\frac{\mu}{2}\|\Tilde{\boldsymbol{\omega}}^{t+1(\dag)}-\Tilde{\boldsymbol{\omega}}^{t(\dag)}\|^{2}\right]$ is upper bounded by
%\begin{align*}
%        \mathbb{E}&\left [\frac{\mu}{2}\|\Tilde{\boldsymbol{\omega}}^{t+1(\dag)}-\Tilde{\boldsymbol{\omega}}^{t(\dag)}\|^{2}\right]\\
%        &\leq \left( \exp{\left(\frac{(3M+2)E_{l}}{(2M+2)(z-2)}\right)} - 1 \right) \frac{(z-2)\mu}{2} \notag \\
%&( \frac{z}{z-1}\eta^{2}\sigma^{2}_{l} + 3z\eta^{2}\sigma^{2}_{g}  \notag \\
%& + 3z\eta^{2}\frac{\sum_{k=1}^{K}\mathbb{E}\|\nabla f(\tilde{\boldsymbol{\omega}}_{k}^{t(\dag)})\|^{2} }{K})
%\end{align*}
\begin{align*}
        \mathbb{E}\left [\frac{\mu}{2}\|\Tilde{\boldsymbol{\omega}}^{t+1(\dag)}-\Tilde{\boldsymbol{\omega}}^{t(\dag)}\|^{2}\right]
        \leq \mu S_{1}( S_{2}+ 3\mathbb{E}\|\nabla f(\Tilde{\boldsymbol{\omega}}^{t(\dag)}\|^{2}),
\end{align*}
where $S_{1}=2\eta^{2}M(M-1)\left( \exp{\left(\frac{(3M+2)E_{l}}{4(M^{2}-1)}\right)} - 1 \right)$ and $S_{2}=\frac{1}{2M-1}\sigma^{2}_{l} + 3(2\sigma^{2}_{g}+\sigma^{2}_{p})$.

\end{lemma}
\begin{proof}
Lemma \ref{lemma2:boundary_local_updates} states the boundary for the local updates. Eq. \eqref{eq:t_tplus1_agg_relation} and Lemma \ref{lemma5} states that the scheduling increases the local epochs $E_{l}$ to $E_{l}^{*}$ with the expectation for the (aggregated) global model. Hence,
\begin{align}\label{lemma4_eq1}
    \mathbb{E}&\left [\|\Tilde{\boldsymbol{\omega}}^{t+1(\dag)}-\Tilde{\boldsymbol{\omega}}^{t(\dag)}\|^{2}\right] \notag\\
    &=\mathbb{E}\left [\|\frac{1}{K}\sum_{k=1}^{K}\Tilde{\boldsymbol{\omega}}_{k}^{t+1(\dag)} - \frac{1}{K}\sum_{k=1}^{K}\Tilde{\boldsymbol{\omega}}_{k}^{t(\dag)}\|^{2}\right]\notag\\
    &\leq \mathbb{E}\left [\frac{1}{K}\sum_{k=1}^{K}\|\Tilde{\boldsymbol{\omega}}_{k}^{t+1(\dag)} - \Tilde{\boldsymbol{\omega}}_{k}^{t(\dag)}\|^{2}\right].
\end{align}
The last inequality holds since Jensen's Inequality where the defined function $\phi(.)=\|.\|^{2}$ is convex.
Therefore, substituting the results of Lemma \ref{lemma2:boundary_local_updates} and changing $E_{l}$ by $E_{l}^{*}+1$, we have a boundary with $\frac{\sum_{k=1}^{K}\mathbb{E}\|\nabla f(\tilde{\boldsymbol{\omega}}_{k}^{t(\dag)})\|^{2} }{K}$. With the triangle inequality again and Assumption \ref{assumption:bounded_var}, 
\begin{align}\label{lemma4_eq2}
    &\frac{\sum_{k=1}^{K}\mathbb{E}\|\nabla f(\tilde{\boldsymbol{\omega}}_{k}^{t(\dag)})\|^{2} }{K} \notag \\
    &\leq     
    \frac{1}{K} \sum_{k=1}^{K} \mathbb{E} \Big( \|\nabla f_{k}(\tilde{\boldsymbol{\omega}}_{k}^{t(\dag)}) - \nabla f(\tilde{\boldsymbol{\omega}}_{k}^{t(\dag)})\|^{2} \notag \\
    &\quad + \|\nabla f(\tilde{\boldsymbol{\omega}}^{t(\dag)})\|^{2} \notag \\
    &\quad + \|\nabla f(\tilde{\boldsymbol{\omega}}_{k}^{t(\dag)}) - \nabla f(\tilde{\boldsymbol{\omega}}_{k}^{t(\dag)})\|^{2} \Big) \notag \\
    &\leq 
    \mathbb{E}\|\nabla f(\tilde{\boldsymbol{\omega}}^{t(\dag)})\|^{2} + \sigma_{p}^{2} + \sigma_{g}^{2},
\end{align}
which finishes the proof.
\end{proof}

\begin{lemma}\label{lemma3}
% Lemma 3 and 4 are in reverse position
Under Assumptions~\ref{assumption:mu_lips}~to~\ref{assumption:bounded_var} and Lemma \ref{lemma4},
$    \mathbb{E}\left[<\nabla f(\Tilde{\boldsymbol{\omega}}^{t(\dag)}),\Tilde{\boldsymbol{\omega}}^{t+1(\dag)}-\Tilde{\boldsymbol{\omega}}^{t(\dag)}>\right]$ is upper bounded by
\begin{align}
-\frac{\eta}{2}\|\nabla f(\Tilde{\boldsymbol{\omega}}^{t(\dag)})\|^{2}-\frac{\eta}{2}\mathbb{E}\|\hat{\mathbf{g}}^{t(\dag)}\|^{2}+\frac{3\mu^{2}\eta^{3}E_{l}^{*}}{M}(\sigma_{l}^{2}+\sigma_{g}^{2})
\end{align}
where $E_{l}^{*}=\frac{3M+2}{2(M+1)}E_{l}$
%where constants $S_{1}=\left( \exp{\left(\frac{(3M+2)E_{l}}{4(M^{2}-1)}\right)} - 1 \right)$, $S_{2}=\frac{2M}{2M-1}\eta^{2}\sigma^{2}_{l} + 6M\eta^{2}(\sigma^{2}_{g}+2\sigma^{2}_{p})$
%and $S_{3}=\frac{(E_{l}^{*}+1)\eta^{2}(G^{2}+\sigma_{l}^{2})}{2}$.
\end{lemma}

\begin{proof}
        According to Eq. \eqref{eq:t_tplus1_agg_relation} and let $\hat{\mathbf{g}}^{t(\dag)}= \frac{\sum_{k=1}^{K}\sum_{\tau=0}^{E_{l}^{*}}\mathbf{g}_{\tau,k}(\Tilde{\boldsymbol{\omega}}^{t(\dag)})}{K}$,

% \begin{align}\label{eq:lemma3_eq1}
%     &\mathbb{E}\left[<\nabla f(\Tilde{\boldsymbol{\omega}}^{t(\dag)}),\Tilde{\boldsymbol{\omega}}^{t+1(\dag)}-\Tilde{\boldsymbol{\omega}}^{t(\dag)}>\right] = \notag\\
%     %&\mathbb{E}\left[<\nabla f(\Tilde{\boldsymbol{\omega}}^{t(\dag)}), -\eta\mathbb{E}\left(\hat{\mathbf{g}}^{t(\dag)}\right)>\right] 
%     & = -\eta \mathbb{E}\left[<\nabla f(\Tilde{\boldsymbol{\omega}}^{t(\dag)}), \mathbb{E}\left(\hat{\mathbf{g}}^{t(\dag)}\right)>\right] 
% \end{align}

$
    \mathbb{E}[<\nabla f(\Tilde{\boldsymbol{\omega}}^{t(\dag)}),\Tilde{\boldsymbol{\omega}}^{t+1(\dag)}-\Tilde{\boldsymbol{\omega}}^{t(\dag)}>] 
    %&\mathbb{E}\left[<\nabla f(\Tilde{\boldsymbol{\omega}}^{t(\dag)}), -\eta\mathbb{E}\left(\hat{\mathbf{g}}^{t(\dag)}\right)>\right] 
= -\eta \mathbb{E}\left[<\nabla f(\Tilde{\boldsymbol{\omega}}^{t(\dag)}), \mathbb{E}\left(\hat{\mathbf{g}}^{t(\dag)}\right)>\right] 
$.
Let's prove the boundary for $-\mathbb{E}\left[<\nabla f(\Tilde{\boldsymbol{\omega}}^{t(\dag)}), \mathbb{E}\left(\hat{\mathbf{g}}^{t(\dag)}\right)>\right]$.
% \begin{align}\label{lemma3_1}
%     &-\mathbb{E}\left[\left<\nabla f(\Tilde{\boldsymbol{\omega}}^{t(\dag)}), \mathbb{E}\left(\hat{\mathbf{g}}^{t(\dag)}\right)\right>\right] \notag \\
%     &\stackrel{(a)}{=} -\frac{1}{2}\|\nabla f(\Tilde{\boldsymbol{\omega}}^{t(\dag)})\|^{2} -\frac{1}{2}\mathbb{E}\|\hat{\mathbf{g}}^{t(\dag)}\|^{2} \notag \\
%     &\quad +\frac{1}{2}\|\nabla f(\Tilde{\boldsymbol{\omega}}^{t(\dag)})-\mathbb{E}[\hat{\mathbf{g}}^{t(\dag)}]\|^{2},
% \end{align}

\begin{align}\label{lemma3_1}
    &-\mathbb{E}[\left<\nabla f(\Tilde{\boldsymbol{\omega}}^{t(\dag)}), \mathbb{E}\left(\hat{\mathbf{g}}^{t(\dag)}\right)\right>]\stackrel{(a)}{=} -\frac{1}{2}\|\nabla f(\Tilde{\boldsymbol{\omega}}^{t(\dag)})\|^{2}\notag \\
    &  -\frac{1}{2}\mathbb{E}\|\hat{\mathbf{g}}^{t(\dag)}\|^{2} +\frac{1}{2}\|\nabla f(\Tilde{\boldsymbol{\omega}}^{t(\dag)})-\mathbb{E}[\hat{\mathbf{g}}^{t(\dag)}]\|^{2}
\end{align}

, where $(a)$ is due to $<\mathbf{a},\mathbf{b}>=\frac{\mathbf{a}^{2}+\mathbf{b}^{2}-(\mathbf{a}-\mathbf{b})^{2}}{2}$. Then only the third term in the right-hand-side is left, since the boundary for the first and second term can be found by Lemma \ref{lemma4}. Based on Assumption \ref{assumption:mu_lips} and denote $t_{c}$ as the start of the communication round of $t$, i.e before training,
\begin{align}\label{lemma3_2}
    &\frac{1}{2}\|\nabla f(\Tilde{\boldsymbol{\omega}}^{t(\dag)})-\mathbb{E}[\hat{\mathbf{g}}^{t(\dag)}]\|^{2} \leq \frac{\mu^{2}}{2K}\sum_{k=1}^{K}\mathbb{E}\|\Tilde{\boldsymbol{\omega}}^{t(\dag)}-\Tilde{\boldsymbol{\omega}}^{t(\dag)}_{k}\|^{2} \notag \\
    &= \frac{\mu^{2}}{2}\sum_{k=1}^{K}\mathbb{E}[\frac{1}{K}\|\Tilde{\boldsymbol{\omega}}^{t_{c}(\dag)}- \frac{1}{K}\sum_{j=1}^{K}\mathbb{E}_{\mathcal{P}_{t,j}}[\frac{\eta}{M}\sum_{j\in \mathcal{P}_{t,j}}\sum_{\tau=t_{c}}^{t_{c}+E_{l}^{*}}\Tilde{g}_{\tau,j}^{t_{c}(\dag)}]  \notag \\
    &\quad  -\Tilde{\boldsymbol{\omega}}^{t_{c}(\dag)} + \mathbb{E}_{\mathcal{P}_{t,k}}[\frac{\eta}{M}\sum_{k\in \mathcal{P}_{t,k}}\sum_{\tau=t_{c}}^{t_{c}+E_{l}^{*}}\Tilde{g}_{\tau,k}^{t_{c}(\dag)}]\|^{2}] \notag \\
    &= \frac{\mu^{2}\eta^{2}}{2}\sum_{k=1}^{K}\mathbb{E}[\frac{1}{K}\|\mathbb{E}_{\mathcal{P}_{t,k}}[\frac{1}{M}\sum_{k\in \mathcal{P}_{t,k}}\sum_{\tau=t_{c}}^{t_{c}+E_{l}^{*}}\Tilde{g}_{\tau,k}^{t_{c}(\dag)}]  \notag \\
    &\quad  -\frac{1}{K}\sum_{j=1}^{K}\mathbb{E}_{\mathcal{P}_{t,j}}[\frac{1}{M}\sum_{j\in \mathcal{P}_{t,j}}\sum_{\tau=t_{c}}^{t_{c}+E_{l}^{*}}\Tilde{g}_{\tau,j}^{t_{c}(\dag)}]\|^{2}] \notag \\
    &=\frac{\mu^{2}\eta^{2}}{2}\sum_{k=1}^{K}\mathbb{E}[\frac{1}{K}\|\mathbb{E}_{\mathcal{P}_{t,k}}[\frac{1}{M}\sum_{k\in \mathcal{P}_{t,k}}\sum_{\tau=t_{c}}^{t_{c}+E_{l}^{*}}(\Tilde{g}_{\tau,k}^{t_{c}(\dag)}-\nabla f_{k}(\Tilde{\boldsymbol{\omega}}^{t_{c}(\dag)}_{k,\tau}))]\notag\\
    &\quad -\frac{1}{K}\sum_{j=1}^{K}\mathbb{E}_{\mathcal{P}_{t,j}}[\frac{1}{M}\sum_{j\in \mathcal{P}_{t,j}}\sum_{\tau=t_{c}}^{t_{c}+E_{l}^{*}}(\Tilde{g}_{\tau,j}^{t_{c}(\dag)}-\nabla f_{j}(\Tilde{\boldsymbol{\omega}}^{t_{c}(\dag)}_{j,\tau}))]\notag \\
    &\quad +\mathbb{E}_{\mathcal{P}_{t,k}}[\frac{1}{M}\sum_{k\in \mathcal{P}_{t,k}}\sum_{\tau=t_{c}}^{t_{c}+E_{l}^{*}}\nabla f_{k}(\Tilde{\boldsymbol{\omega}}^{t_{c}(\dag)}_{k,\tau})]\notag\\
    &\quad -\frac{1}{K}\sum_{j=1}^{K}\mathbb{E}_{\mathcal{P}_{t,j}}[\frac{1}{M}\sum_{j\in \mathcal{P}_{t,j}}\sum_{\tau=t_{c}}^{t_{c}+E_{l}^{*}}\nabla f_{j}(\Tilde{\boldsymbol{\omega}}^{t_{c}(\dag)}_{j,\tau})]\|^{2}].\notag \\
    &\stackrel{(b)}{\leq} \frac{3\mu^{2}\eta^{2}}{2K} \sum_{k=1}^{K}[ \mathbb{E}_{\mathcal{P}_{t,j}} \| \frac{1}{M} \sum_{k\in \mathcal{P}_{t,k}} \sum_{\tau=t_{c}}^{t_{c}+E_{l}^{*}} (\Tilde{g}_{\tau,k}^{t_{c}(\dag)}  \notag \\
    &\qquad  - \nabla f_{k}(\Tilde{\boldsymbol{\omega}}^{t_{c}(\dag)}_{k,\tau}) ) \|^{2} \notag \\
    &\quad + \frac{1}{K}\sum_{j=1}^{K}\mathbb{E}_{\mathcal{P}_{t,j}}[\|\frac{1}{M}\sum_{j\in \mathcal{P}_{t,j}}\sum_{\tau=t_{c}}^{t_{c}+E_{l}^{*}}\nabla f_{j}(\Tilde{\boldsymbol{\omega}}^{t_{c}(\dag)}_{j,\tau})\|^{2}] \notag\\
    &\quad +.\|\mathbb{E}_{\mathcal{P}_{t,k}}[\frac{1}{M}\sum_{k\in \mathcal{P}_{t,k}}\sum_{\tau=t_{c}}^{t_{c}+E_{l}^{*}}\nabla f_{k}(\Tilde{\boldsymbol{\omega}}^{t_{c}(\dag)}_{k,\tau})] \notag\\
    &\qquad -\frac{1}{K}\sum_{j=1}^{K}\mathbb{E}_{\mathcal{P}_{t,j}}[\frac{1}{M}\sum_{j\in \mathcal{P}_{t,j}}\sum_{\tau=t_{c}}^{t_{c}+E_{l}^{*}}\nabla f_{j}(\Tilde{\boldsymbol{\omega}}^{t_{c}(\dag)}_{j,\tau})]\|^{2}] ]\notag\\
   &\stackrel{(c)}{\leq}  \frac{3\mu^{2}\eta^{2}}{2K}[\frac{E_{l}^{*}}{M}\sigma_{l}^{2}+\frac{E_{l}^{*}K\sigma^{2}_{l}}{KM}+\frac{E_{l}^{*}}{M}\sigma_{g}^{2}+\frac{E_{l}^{*}K\sigma^{2}_{g}}{KM}]\notag\\
   & =  \frac{3\mu^{2}\eta^{2}E_{l}^{*}}{M}(\sigma_{l}^{2}+\sigma_{g}^{2})
\end{align}
In the above formulation, the variable $j$ serves to distinguish from $k$, ensuring clarity in the representation of individual client contributions. Here, $\mathcal{P}_{t,k}$ denotes the selection probability of client $k$ at time $t$. The inequality marked as $(b)$ derives from the application of the Cauchy-Schwarz inequality, exemplified by the relation $\|a+b+c\|^{2} \leq 3(\|a\|^{2}+\|b\|^{2}+\|c\|^{2})$. The step labeled as $(c)$ leverages a similar analytical technique, focusing on the aggregation of global gradients, thereby facilitating the derivation of $\sigma^{2}_{g}$, which is in conjunction with Assumption \ref{assumption:bounded_var}.
Substitute the results from Lemma \ref{lemma4} to Eq.\eqref{lemma3_1} with multiplying $\eta$ to conclude the proof for Lemma \ref{lemma3}.

\end{proof}

\subsection{Proof for Theorem 1}\label{appendix_theorem1}
Combining the bounds from Lemma \ref{lemma4} and Lemma \ref{lemma3}, we have:
\begin{align}\label{eq:theorem1}
    &\mathbb{E}[f(\Tilde{\boldsymbol{\omega}}^{t+1(\dag)})] -\mathbb{E}[f(\Tilde{\boldsymbol{\omega}}^{t(\dag)})] \leq \notag \\
    &\quad \mathbb{E}[\frac{\mu}{2}\|\Tilde{\boldsymbol{\omega}}^{t+1(\dag)}-\Tilde{\boldsymbol{\omega}}^{t(\dag)}\|^{2}] \\
    &+ \mathbb{E}[\langle \nabla f(\Tilde{\boldsymbol{\omega}}^{t(\dag)}), \Tilde{\boldsymbol{\omega}}^{t+1(\dag)}-\Tilde{\boldsymbol{\omega}}^{t(\dag)} \rangle] \notag\\
    &= -\frac{\eta}{2}\|\nabla f(\Tilde{\boldsymbol{\omega}}^{t(\dag)})\|^{2}+\frac{\mu-\eta}{2}\mathbb{E}\|\hat{\mathbf{g}}^{t(\dag)}\|^{2}+\frac{3\mu^{2}\eta^{3}E_{l}^{*}}{M}(\sigma_{l}^{2}+\sigma_{g}^{2})\\
    &\leq \frac{6S_{1}(\mu-\eta)-\eta}{2}\|\nabla f(\Tilde{\boldsymbol{\omega}}^{t(\dag)})\|^{2}+(\mu-\eta) S_{1}S_{2} + \notag \\
    &\quad \frac{3\mu^{2}\eta^{3}(\frac{3M+2}{2(M+1)}E_{l}-1)}{M}(\sigma_{l}^{2}+\sigma_{g}^{2}).
\end{align}

It is the fact that $\min{\|\nabla f(\Tilde{\boldsymbol{\omega}}^{t(\dag)}\|^{2}}\leq \frac{\sum_{t=0}^{T-1}\|\nabla f(\Tilde{\boldsymbol{\omega}}^{t(\dag)}\|^{2}}{T}$. Summing up Eq. \eqref{eq:theorem1} from $t=0$ to a large $t=T$ concludes the proof. In detail, Given the inequality for each iteration \(t\):
\begin{equation}
\begin{aligned}
    &\mathbb{E}[f(\Tilde{\boldsymbol{\omega}}^{t+1(\dag)})] - \mathbb{E}[f(\Tilde{\boldsymbol{\omega}}^{t(\dag)})] \leq \frac{6S_{1}(\mu-\eta)-\eta}{2}\|\nabla f(\Tilde{\boldsymbol{\omega}}^{t(\dag)})\|^{2} \\
    &+ (\mu-\eta) S_{1}S_{2} + \frac{3\mu^{2}\eta^{3}\left(\frac{3M+2}{2(M+1)}E_{l}-1\right)}{M}(\sigma_{l}^{2}+\sigma_{g}^{2}).
\end{aligned}
\end{equation}

Summing this inequality from \(t=0\) to \(t=T-1\) yields:
\begin{equation}
\begin{aligned}
    &\mathbb{E}[f(\Tilde{\boldsymbol{\omega}}^{T(\dag)})] - \mathbb{E}[f(\Tilde{\boldsymbol{\omega}}^{0(\dag)})] \\
    & \leq\sum_{t=0}^{T-1} \left( \frac{6S_{1}(\mu-\eta)-\eta}{2}\|\nabla f(\Tilde{\boldsymbol{\omega}}^{t(\dag)})\|^{2} \right) + T \cdot \left[ (\mu-\eta) S_{1}S_{2} \right. \\
    &\quad \left. + \frac{3\mu^{2}\eta^{3}\left(\frac{3M+2}{2(M+1)}E_{l}-1\right)}{M}(\sigma_{l}^{2}+\sigma_{g}^{2}) \right].
\end{aligned}
\end{equation}

To isolate the cumulative gradient norm terms across \(T\) iterations, we divide the inequality by the coefficient of the gradient norm term:
\begin{equation}
\begin{aligned}
    \sum_{t=0}^{T-1} \|\nabla f(&\Tilde{\boldsymbol{\omega}}^{t(\dag)})\|^{2} \\
    &\leq \frac{2}{\eta-6S_{1}(\mu-\eta)} \left( \mathbb{E}[f(\Tilde{\boldsymbol{\omega}}^{0(\dag)})] - \mathbb{E}[f(\Tilde{\boldsymbol{\omega}}^{T(\dag)})] \right) \\
    &\quad + S_{3}T
\end{aligned}
\end{equation}

, where 
$S_{3}=\frac{2}{\eta-6S_{1}(\mu-\eta)} \cdot \left[ (\mu-\eta) S_{1}S_{2} + \frac{3\mu^{2}\eta^{3}(3M+2)E_{l}}{2(M+1)M}(\sigma_{l}^{2}+\sigma_{g}^{2}) \right]$. When T is large, The denominator $\eta-6S_{1}(\mu-\eta)$ is nominated by $\eta$. Then when we choose $\eta\propto\mathcal{O}(\frac{1}{\sqrt{T}\mu})$ and as $T$ is large enough, $S_{3}$ diminishes closely to $0$.

\bibliography{ppfl}
\bibliographystyle{IEEEtran}


\newpage

\end{document}